\documentclass[conference]{IEEEtran}
\IEEEoverridecommandlockouts

\usepackage{cite}
\usepackage{amsmath,amssymb,amsfonts}
\usepackage[ruled]{algorithm2e}
\usepackage{graphicx}
\usepackage{textcomp}
\usepackage{xcolor}
\usepackage{enumitem}

\usepackage{hyperref}
\usepackage{array}          
\usepackage{algpseudocode}  
\usepackage{subfigure}      

\usepackage{multirow}       
\usepackage{amsthm}         
\usepackage{pifont}
\usepackage{booktabs}       

\usepackage{float}          

\def\BibTeX{{\rm B\kern-.05em{\sc i\kern-.025em b}\kern-.08em
    T\kern-.1667em\lower.7ex\hbox{E}\kern-.125emX}}

\begin{document}

\title{Heterogeneous Graph Neural Network for Privacy-Preserving Recommendation}

\author{\IEEEauthorblockN{
Yuecen Wei\IEEEauthorrefmark{2}\IEEEauthorrefmark{3},
Xingcheng Fu\IEEEauthorrefmark{4}\IEEEauthorrefmark{5},
Qingyun Sun\IEEEauthorrefmark{4}\IEEEauthorrefmark{5},
Hao Peng\IEEEauthorrefmark{4},
Jia Wu\IEEEauthorrefmark{6},
Jinyan Wang\IEEEauthorrefmark{1}\IEEEauthorrefmark{2}\IEEEauthorrefmark{3} and
Xianxian Li\IEEEauthorrefmark{1}\IEEEauthorrefmark{2}\IEEEauthorrefmark{3}}

\IEEEauthorblockA{\IEEEauthorrefmark{2}Guangxi Key Lab of Multi-source Information Mining \& Security, Guangxi Normal University, Guilin, China}
\IEEEauthorblockA{\IEEEauthorrefmark{3}School of Computer Science and Engineering, Guangxi Normal University, Guilin, China}
\IEEEauthorblockA{\IEEEauthorrefmark{4}Beijing Advanced Innovation Center for Big Data and Brain Computing, Beihang University, Beijing, China}
\IEEEauthorblockA{\IEEEauthorrefmark{5}School of Computer Science and Engineering, Beihang University, Beijing, China}
\IEEEauthorblockA{\IEEEauthorrefmark{6}School of Computing, Macquarie University, Sydney, Australia} 
 Email:
 weiyc@stu.gxnu.edu.cn,
\{fuxc,sunqy,penghao\}@act.buaa.edu.cn, \\
 jia.wu@mq.edu.au, 
\{wangjy612,lixx\}@gxnu.edu.cn
 
 \thanks{\IEEEauthorrefmark{1}Corresponding author.}
}

\maketitle

\begin{abstract}
Social networks are considered to be heterogeneous graph neural networks (HGNNs) with deep learning technological advances. 
HGNNs, compared to homogeneous data, absorb various aspects of information about individuals in the training stage. 
That means more information has been covered in the learning result, especially sensitive information. 
However, the privacy-preserving methods on homogeneous graphs only preserve the same type of node attributes or relationships, which cannot effectively work on heterogeneous graphs due to the complexity. 
To address this issue, we propose a novel heterogeneous graph neural network privacy-preserving method based on a differential privacy mechanism named \textbf{HeteDP}, which provides a double guarantee on graph features and topology. 
In particular, we first define a new attack scheme to reveal privacy leakage in the heterogeneous graphs. 
Specifically, we design a two-stage pipeline framework, which includes the privacy-preserving feature encoder and the heterogeneous link reconstructor with gradients perturbation based on differential privacy to tolerate data diversity and against the attack. 
To better control the noise and promote model performance, we utilize a bi-level optimization pattern to allocate a suitable privacy budget for the above two modules. 
Our experiments on four public benchmarks show that the HeteDP method is equipped to resist heterogeneous graph privacy leakage with admirable model generalization.

\end{abstract}

\begin{IEEEkeywords}
privacy-preserving, recommendation, differential privacy, heterogeneous graph
\end{IEEEkeywords}

\section{Introduction}
The heterogeneous graph is an extraordinary information network, which consists of multiple node types and multiple relation types~\cite{pengCIKM2019}.
Social relations are one of the networks that are most complex and closest to people's lives. 
According to their interactions and inter-dependencies, recommendation predicts the products the user will purchase while inferring the user's implicit tendency~\cite{Diifnet++,GNNRes2019}. 
Therefore, heterogeneous information networks (HINs)~\cite{hg1-2020} are widely used in recommender systems due to their enriched heterogeneous data. 
For example, in movie recommendation, entities have not only users and movies but also stores, and the relationship has collections in addition to purchases~\cite{SurveyGNNRecSys2021}, etc. 
For adapting the non-Euclidean structure of HINs, existing works leverage high-level information~\cite{hgmagnn-2020,HGT2020,pengTKDE2021,HINCIKM2021,HINAAAI2021} by other platforms sharing (e.g., logging in with a third-party account)~\cite{HighOrderRec2020,RecAgainstAI2021,SocialNetworkPrivacy2020,HypergraphCNRec2021} or semantic-level information from multiple entities~\cite{HIN2Vec2017, metapath2vec2017, pengCIKM2019}. 
In this way, these works always fuse the social network data and other side information of the users and items as a unified heterogeneous graph to improve model performance. 
However, while HINs boost recommendation capabilities, they also bring an additional risk of privacy leakage. 

\begin{figure}[t]
	\centering
	\includegraphics[width=0.48\textwidth]{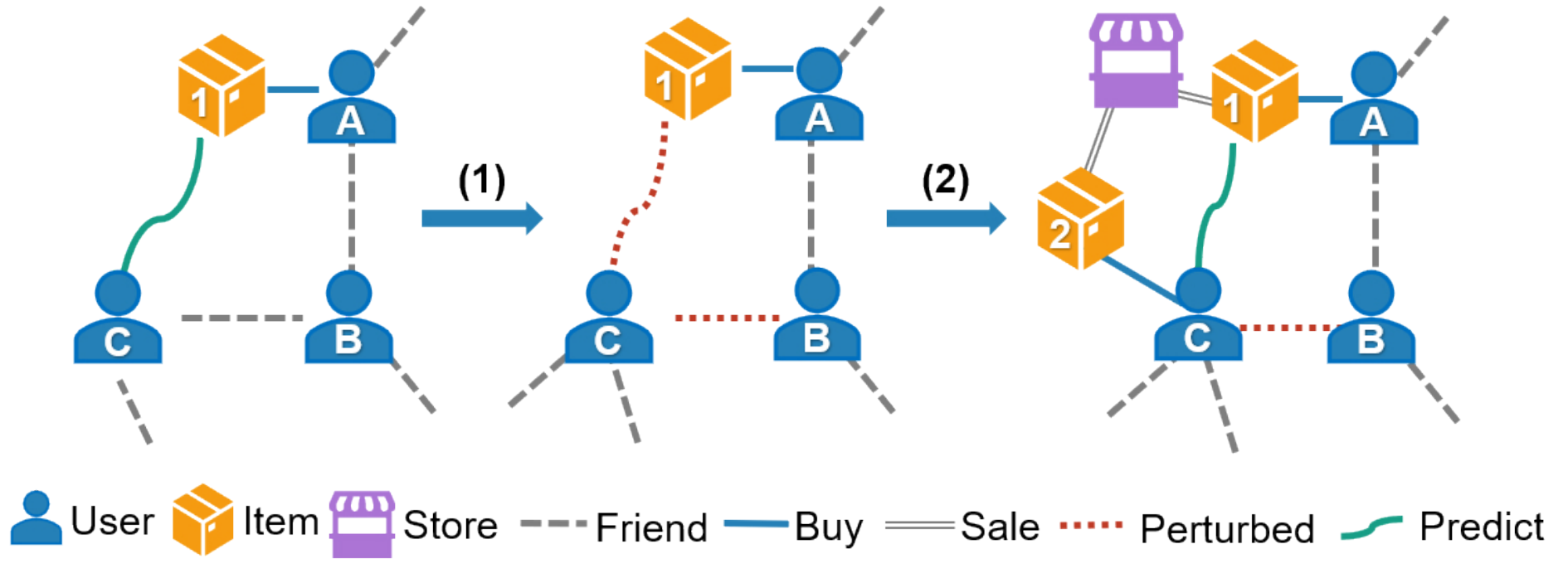}
    \vspace{-1em}
	\caption{An example of privacy risk from a homogeneous graph to a heterogeneous. 
	Change $(1)$ represents general privacy-preserving measures for nodes on the homogeneous graph. 
    Change $(2)$ indicates that the former method has a poor protection effect on heterogeneous graphs because more node types are considered. 
}
\vspace{-2em}
\label{example_1}
\end{figure}
Graph neural networks (GNNs) are widely used to implement heterogeneous graph learning and achieve remarkable results, as a popular and powerful graph representation model~\cite{sun1www21,AGE2021arxiv,sun2aaai2022,pengTNNLS2022,sun3CIKM2022}, such as recommended systems~\cite{res2-2019,res3-2019,pengSIGIR2020,hg2-2020}. 

However, most existing works focus on how to improve the representational power of graphs and ignore the security issues of sensitive information in graph data. 
For user privacy, some non-Euclidean data may more intuitively discover the relationships between users and some sensitive information~\cite{pengCIKM2022}, such as social relationships~\cite{SocialNetworkPrivacy2020}, behavioral trajectories~\cite{trajectoriesProtect2021,locationAndSemanticPrivacy2021,locationAndSemanticDP2019}, and medical records. 
While people benefit from the convenience of the recommendation, they are faced with recorded behavior data and learned and used all aspects of information that would bring a series of privacy leakage risks. 
In the real social world, some malicious people can obtain individuals' sensitive characteristics from enriching recommendations~\cite{SurveyGNNRecSys2021}, such as identification and phone number, address, and even social relationships. 
The privacy leakage risk of this heterogeneous information is reflected in both feature and topology levels. 

Recently, to address privacy problems in graph data, some existing works focus on privacy leakage in graph-based~\cite{obfuscationSemantic2021,SurveyHereData2021}. 
Differential privacy~\cite{RecAgainstAI2021,DPGGAN2021} based on data distribution perturbation, as advanced privacy-preserving technology, is widely used in deep learning because of the strict mathematical definition. 
Therefore, there is a remarkable limitation: the privacy-preserving method of a homogeneous graph cannot solve the problem caused by heterogeneity. 
For example, different types of nodes may no longer be independent of each other in features and topology but have semantic dependencies. 
On the one hand, \figurename~\ref{example_1} illustrates an inference and preservation between homogeneous and heterogeneous graphs. 
The model predicts user $C$ will buy the item by neighbor relationships. 
Specifically, the inference is drawn due to $A$'s historical shopping record, and $B$ is a neighbor of both $A$ and $C$. 
The existing works protect the direct relationship between users by disturbing their links, reducing the predicted probability. 
However, different types of nodes and edges exist in heterogeneous social networks, respectively.
When other node types are considered on the graph, we still can infer the buying action since the existing homogeneous graph methods only pay attention to the influence between the same node types. 
On the other hand, we are assuming that a malicious attacker can compare the target network with another network whose topology is similar and public. 
The background knowledge allows him to obtain connections between arbitrary nodes regardless of node types and analyze the semantics to understand the preferences of a particular user.
\figurename~\ref{example_2} shows that user $A$ is the only node in the subgraph with degree two and in one quadrilateral and business $B$ is the node with degree four and in two quadrilaterals. 
The attacker can utilize the topology of heterogeneous graphs to infer that $A$ has purchased items from $B$. 
If we could change the links between nodes while keeping the topological properties of each node as much as possible, the edges in the graph would not be directly exposed. 
The examples of the above attack methods show that the traditional naive differential privacy based on the I.I.D assumption is difficult to apply to the heterogeneous graphs of non-I.I.D directly. 

\begin{figure}[t]
	\centering
	\includegraphics[width=0.48\textwidth]{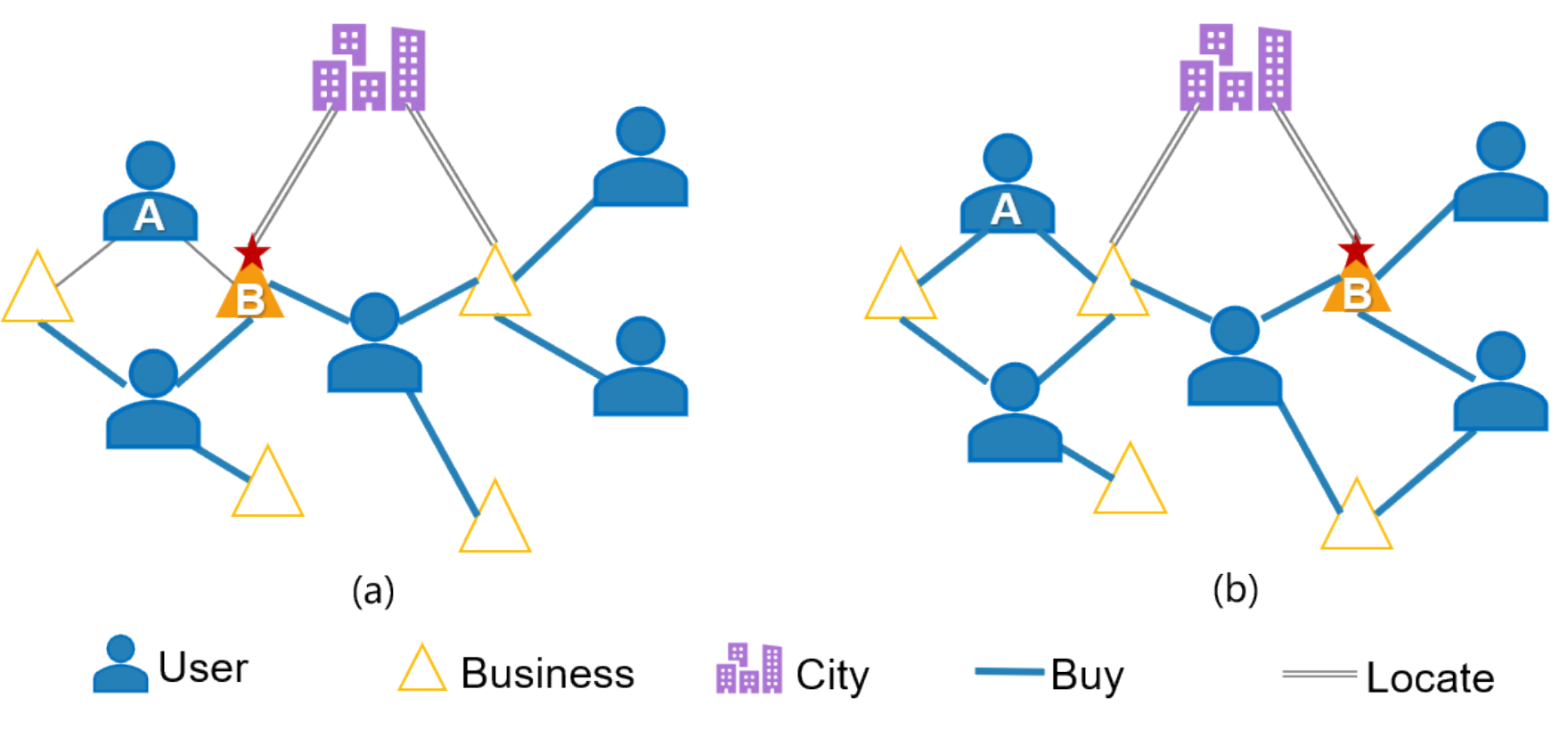}
    \vspace{-1em}
	\caption{An example of privacy risk from topology properties and topology protection. 
	$(a)$ The original graph structure. 
	$(b)$ The topological structure after perturbation. 
}
\vspace{-2em}
\label{example_2}
\end{figure}

Consequently, the core issue is, ``\textit{Can we put forward a heterogeneous graph neural network for privacy-preserving recommendation model which is able to adapt to the heterogeneity of graph data with resisting the `betrayal' of graph topology and different neighbors?}'' 
In conclusion, there are three extraordinary challenges in heterogeneous graphs about privacy-preserving: 
$(1)$ privacy is leaked through different types of higher-order neighbor information; 
$(2)$ even if the topology of homogeneous nodes is changed, privacy can still reveal the relationship between the same node types through high-level semantics;  
$(3)$ the difficulty lies in how to trade off privacy guarantees and compelling predictions. 

To resolve above problem, we propose a novel \underline{\textbf{Hete}}rogeneous Graph Neural Network Privacy-Preserving method based on \underline{\textbf{D}}ifferential \underline{\textbf{P}}rivacy named \textbf{HeteDP}\footnote{The source code is released at \url{https://github.com/AixWinnie/HeteDP}. }. 
First, we define a novel privacy leakage scenario for heterogeneous graph recommendations, and further, we reveal the privacy leakage risks associated with the heterogeneity of heterogeneous graphs. 
Specifically, We designed two stages of DP strategies to guarantee the privacy of graph features and topology for the privacy leakage problem of the heterogeneous graph. 
We propose a reasonable feature perturbation method based on a heterogeneous attention mechanism to encode the node representations. 
The sensitivity of features' Gaussian noise is learned by the neighbor influence and relationship influence of nodes under different relational subgraphs. 
Then, we input the perturbed node representations to a variational graph auto-encoders (VGAE)~\cite{VGAE2016} of the heterogeneous graph for reconstructing the privacy-preserving topology. 
The reconstructor can set learnable gradient clipping hyperparameters as noise sensitivity to clip and perturb the gradients. 
In addition, to solve the privacy budget allocation problem of global differential privacy, we design a bi-level optimization algorithm for HeteDP. 
We summarize our main contributions as follows: 
\begin{itemize}
    \item Aiming at the nature of heterogeneous information networks, we define a novel privacy leakage scenario and reveal privacy leakage risks for heterogeneous graph recommendations. 
    \item We propose a novel unsupervised privacy-preserving learning framework, named Heterogeneous Graph Neural Network Privacy-Preserving with Differential Privacy (HeteDP). 
    HeteDP is a two-stage pipeline framework, which can preserve the privacy of the feature and topology of the heterogeneous graph. 
    \item We design a adaptive privacy budget allocation by using bi-level optimization to balance the privacy and utility of HeteDP. 
    \item Experiments demonstrate the adaptability and generalization performance of the model on multiple real-world datasets. 
    We further analyze the necessity of each part of HeteDP and the feasibility of the whole model in detail. 
\end{itemize}

\section{Related Work}\label{sec:relatedwork}
\subsection{Heterogeneous Graph Neural Network}
HGNNs~\cite{HetGNN2019,RGCN2018,GCMC2017} are a powerful representation learning method with outstanding generalization ability. 
Existing models can fully use intricate information in heterogeneous networks to learn more inner information and improve model performance. 
RGCN~\cite{RGCN2018} and HetGNN~\cite{HetGNN2019} focus on information aggregation on multi-relational heterogeneous graphs, using weight matrices and random walks to aggregate information with different neighbors. 
Metapath2vec~\cite{metapath2vec2017} and HIN2Vec~\cite{HIN2Vec2017} learn node representations based on meta-path random walks to incorporate semantic information in heterogeneous graphs. 
With the contextualization of advanced research, many works have also made excellent progress in recommendations. 
DiffNet++~\cite{Diifnet++} uses attention to fuse the user's node neighbor and interest preference to obtain the embedding of users and items. 
RecoGCN~\cite{RecoGCN2019} is a relation-aware GNN that aggregates embeddings on meta-paths by computing semantic weights with an attention mechanism. 
However, the presentation and application of abundant data will undoubtedly expose more user information, which is more conducive to attackers maliciously inferring and obtaining sensitive user data. 

\subsection{Graph Privacy-Preserving}
As GNNs play a vital role in deep learning, the privacy issue in graph representation learning is exposed. 
Some early works attempted to preserve the privacy of graph data and achieved meaningful results. 
These works preserve users with personalized privacy-preserving~\cite{PersonalPrivacy2010} and leverage anonymization mechanisms to prevent attackers from inferring sensitive information. 
Recently, DPGGAN~\cite{DPGGAN2021} has performed differential privacy in GNNs by referring to DP-SDG~\cite{DPSGD2016} privacy-preserving design patterns and taking advantage of VGAE~\cite{VGAE2016}. 
GERAI~\cite{RecAgainstAI2021} is a recommendation model which combines GCN and DP to ensure the utility of the learning model while preventing users from attribute inference attacks. 
To improve the utility of privacy-preserving models, it is necessary to personalize the privacy budget for different types of data and reduce unnecessary noise injection~\cite{HDPviaGraph2022}. 

However, with the addition of more side information, the inference capability of the attackers may be enhanced, and the existing methods are difficult to adapt to the diversity of the heterogeneous graph.

\section{Preliminaries and Problem Definition}
Differential privacy~\cite{DP2006} is recognized as one of the quantifiable and practical privacy-preserving models. 
The basic idea is that any computation cannot be significantly affected by any operation such as add, delete and modify. 
Even if the attackers know all records except this one, they cannot obtain any information from it. 
Two adjacent datasets $D$ and ${D}' $ differ by at most one record and are defined as follows. 

\textbf{$\left ( \epsilon ,\delta  \right ) $-Differential Privacy~\cite{DP20062}. }
A random algorithm $\mathcal{M} $ satisfies $\left ( \epsilon ,\delta  \right ) $-Differential Privacy for any two neighboring data sets $D$ and $D^{\prime }$ and any possible subset of output 
$\mathcal{O} \subseteq Range\left ( \mathcal{M}   \right ) $, and it holds that
\begin{equation}
\begin{aligned}
    \mathrm{Pr} \left [ \mathcal{M} \left ( D \right ) \in \mathcal{O} \right ] \le e^{\epsilon } \mathrm{Pr} \left [ \mathcal{M} \left ( D^{\prime }  \right ) \in \mathcal{O} \right ] +  \delta. 
\end{aligned}
\end{equation}
The privacy strength of DP increases as the privacy budget decreases,  which is controlled by $\epsilon$ and $\delta$. 
Thus, $\left ( \epsilon ,\delta  \right ) $-DP is guaranteed by adding appropriate noise to the output of the algorithm, and the amount of injected noise is calibrated to the sensitivity. 

\textbf{Sensitivity~\cite{DP20062}. }
Given any query $\mathcal{S} $ on $D$, the sensitivity for any neighboring data sets $D$ and $D^{\prime }$ which is defined as
\begin{equation}\label{eq:defineSens}
\begin{aligned}
    \Delta _{2} \mathcal{S} =\max_{D,D^{\prime } } \left \| \mathcal{S} \left ( D \right ) - \mathcal{S} \left ( D^{\prime }  \right )   \right \|_{2} .
\end{aligned}
\end{equation}

\textbf{Gaussian Mechanism~\cite{GaussianDP2014}. }
Let $\mathcal{S} : D \to \mathbb{O} ^{\mathcal{K} } $ be an arbitrary $\mathcal{K}$-dimensional function and define its $l_{2}$ sensitivity to be $\Delta _{2} \mathcal{S}$. 
The Gaussian Mechanism with parameter $\sigma $ adds noise scaled to $\mathcal{N} \left ( 0,\sigma ^{2}  \right ) $ to each of the $\mathcal{K}$ components of the output. 
Given $\epsilon \in \left ( 0,1 \right ) $ be arbitrary, the Gaussian Mechanism is $\left ( \epsilon ,\delta  \right ) $-DP with
\begin{equation}
\begin{aligned}
    \sigma \ge \sqrt{2\ln_{}{\left ( 1.25/\delta  \right ) } } \Delta _{2} \mathcal{S}/\epsilon .
\end{aligned}
\end{equation}

Adding noise is the primary means to implement privacy-preserving by differential privacy. 
In this work, we will apply Gaussian noise to the node features and link prediction gradients of the heterogeneous graph $G$, respectively, and the overall form is defined as
\begin{equation}\label{eq:GaussionNoise}
\begin{aligned}
    \mathcal{M} \left ( G \right ) \overset{\triangle }{=} \mathcal{S} \left ( G \right ) +\mathcal{N} \left ( 0,\left ( \triangle _{2}\mathcal{S} \right )^2 \sigma^2   \right ) ,
\end{aligned}
\end{equation}
where $\Delta _{2} \mathcal{S}$ controls the amount of noise in the generated Gaussian distribution from which we will sample noise into the target. 

\textbf{Privacy Risk Analysis.} In most social networks, the data is non-I.I.D because the message passing between the information causes them to be interdependent and interact with each other, so the method based on a series of assumptions in which the data is I.I.D is no longer applicable in our scenario. 
The existing works~\cite{DPGGAN2021} only consider that the friends may influence a node at high levels, and they usually reduce the probability of malicious attackers stealing user interest orientations by perturbing the edges between nodes. 
However, we take many aspects of information in the non-Euclidean graph data, which increases the complexity of the data, so that the attacker can obtain the user's preferences by inferring the semantic information between nodes from other node types.  
Consequently, the existing protection methods are challenging to take effect in heterogeneous graphs. 
So our privacy-preserving objects are the graph's sensitive node features and topology structure. 

To summarize the above privacy leakage problem of heterogeneous graphs, we can transform the privacy problem on heterogeneous graphs subject into an associative differential privacy problem of edges with solid semantic correlation. 
This means that our problem further becomes a multi-objective optimization problem for representation learning as well as optimal privacy budget allocation. 

\textbf{Problem Definition.} We aim to maximize privacy-preserving while minimizing information loss due to the noise. 
Therefore, we combine optimal privacy budget allocation with HeteDP optimization as a multi-objective optimization problem. 
There is a heterogeneous graph $G=\left ( V,E,\phi ,\psi  \right ) $ with an entity mapping function $\phi \left ( v \right ) : V\to A$ and a relation mapping function $\psi  \left ( e \right ) : E\to R$, where $V$ and $E$ are the set of nodes and edges. 
Each node $v\in V$ belongs to the node typeset $A$, and each edge $e\in E$ belongs to the edge typeset $R$. 
The graph has the meta-paths $m= a_{1} \overset{r_{1} }{\rightarrow} a_{2} \overset{r_{2} }{\rightarrow}\dots \overset{r_{N-1} }{\rightarrow}a_{N} $ constructed by nodes $a_{i} \in A\left ( i= 1,2,\dots ,N \right ) $ and edges $r_{i} \in R\left ( i= 1,2,\dots ,N \right ) $  , where $a_{i} = \phi \left ( v_{i}  \right ) $ and $r_{i} = \psi  \left ( e_{i}  \right ) =\psi  \left ( \left \langle v_{i},v_{i+1} \right \rangle   \right )$. 
Then, given an objective function with a node feature privacy-preserving learning $T\left ( x,y \right )$ and graph topology privacy-preserving learning $f\left ( x,y \right )$ on the privacy budget of $\epsilon _{f}$ and $\epsilon _{s}$, the problem can be defined as follows
\begin{equation}\label{eq:problem}
\begin{aligned}
    \min_{x\in \epsilon _{s} } T\left ( x,y \right ) ,~\mathbf{s.t.}~y\in F\left ( x \right ) ,
\end{aligned}
\end{equation}
where $F\left ( x \right ) =\mathrm{arg} \min_{y\in \epsilon _{f} }f\left ( x,y \right )  $ and global privacy budget $\epsilon = \epsilon _{f}+\epsilon _{s}$. 
Such problems are usually difficult to find a unified optimal solution, which is the same as the multi-objective optimization in existing graph learning.
We are inspired by the multi-head attention mechanism~\cite{Attention2017} and differentially private stochastic gradient descent~\cite{DPSGD2016}. 
We formulate two protection strategies for node and topology, respectively. 
In particular, we statute the recommendation problem on heterogeneous graphs to an edge prediction problem on graph topology. 
We show in \figurename~\ref{example_2} that the attacker cannot confidently infer that user $A$ has ever shopped in business $B$ with privacy-preserving, preventing him from guessing the user's interest. 
In the next section, we will specify our proposed privacy-preserving approach.

\begin{figure*}[t]
    \centering
    \includegraphics[width=1\linewidth, trim=0 15 0 40, clip]{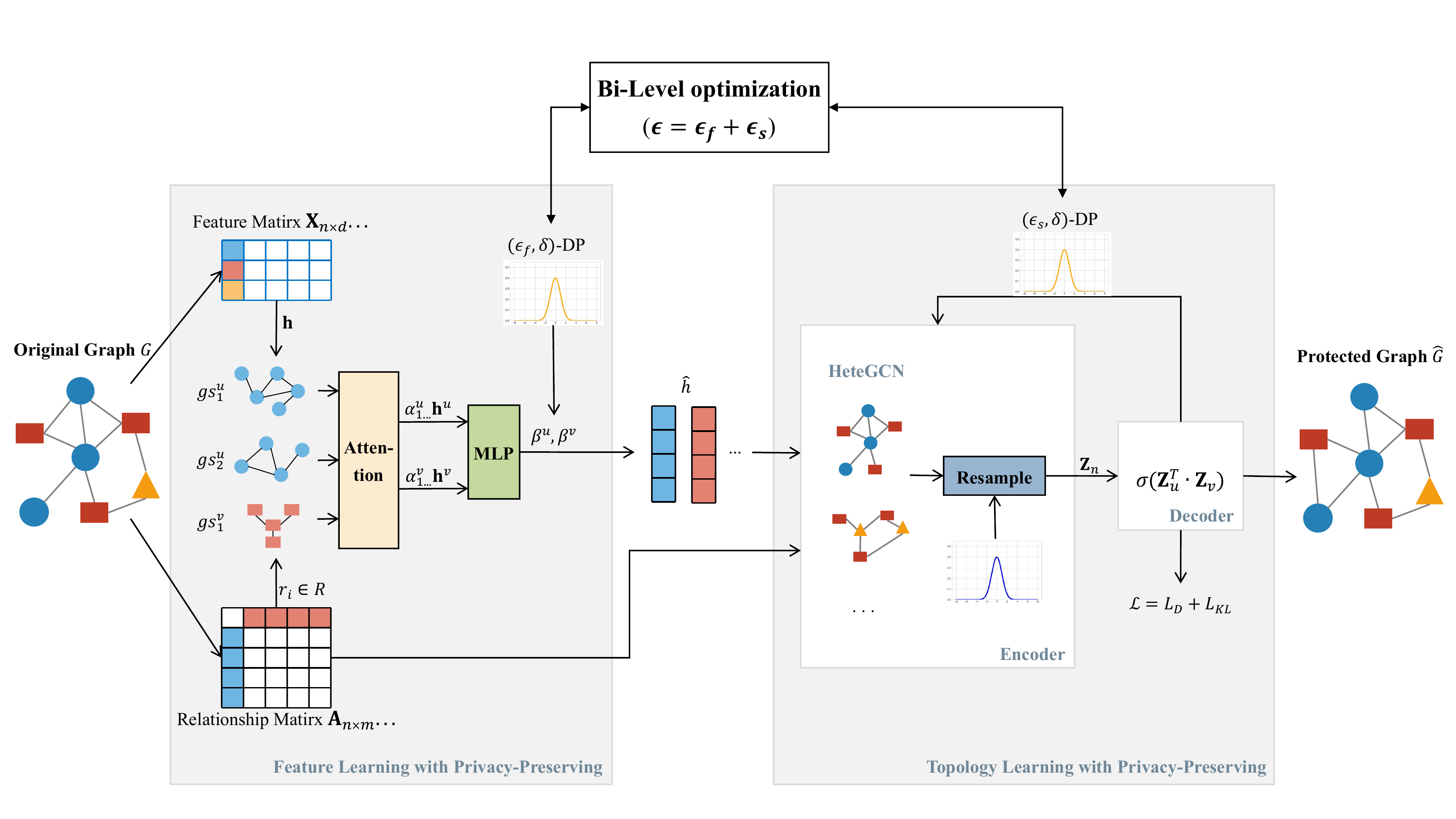}
    \hspace{20pt}
    \vspace{-2em}
    \caption{The framework of HeteDP. 
HeteDP consists of two major components: the privacy-preserving of feature learning and structure learning. 
The first part secures the node attributes, and the second part protects the graph topology. 
The two parts are constrained by a global privacy budget so that the perturbation to the model is within a reasonable range and the optimal accuracy is pursued. }
\label{framework}
\vspace{-1em}
\end{figure*}

\section{Proposed Methodology}
In this section, we introduce an overall learning framework of HeteDP, a heterogeneous graph neural network privacy-preserving with differential privacy, and show how to preserve individuals' privacy. 
\figurename~\ref{framework} shows our privacy-preserving and representations learning framework with the two aspects of DP strategies, where they perform privacy-preserving of node features and graph topology in heterogeneous graphs. 

\subsection{Feature Learning with Privacy-Preserving}\label{subsec:featlearning}
In this section, we detail node-level privacy-preserving and incorporate it into feature learning.
Since the nodes on the graph are affected by their neighbors and the semantic representation, we reflect the importance of various nodes by learning the influence weights of neighbors. 

For the subgraph $gs_m$ generated based on the all walks of each meta-path $m$, we map each node to a uniform space through linear transformation to get the embedding of the $l$-th layer neural network as
\begin{equation}
\begin{aligned}
    \mathbf{z}_{u}^{l} =w_{1} \mathbf{h}_{u}^{l}, 
\end{aligned}
\end{equation}
where $\mathbf{z}_{u}^{l}$ and $\mathbf{h}_{u}^{l}$ are the embedding and the original feature of the node $u$. 

To learn the degree of dependence between node $u$ and its neighbor node $v$, we leverage the attention mechanism and normalize the overall attention value to quantify that we calculate the attention score between nodes as
\begin{equation}\label{eq:nodeAttenScore}
\begin{aligned}
    W_{\left ( uv;m \right ) }^{l } =\textsc{Softmax} \left ( \textsc{Atten} \left (\mathbf{h}_{u}^{l} ,\mathbf{h}_{v}^{l}   \right ) ;m \right ). 
\end{aligned}
\end{equation}
Then, we introduce multi-head attention for node representation learning to pay attention to more aspects and comprehensive neighbor information. We also explicitly obtain the influence weight of node $u$ by other nodes simultaneously. 
So we obtain the multi-head attention coefficients and node representations between nodes on the $(l+1)$-th layer of each subgraph as
\begin{equation}\label{eq:nodeAttenCof}
\begin{aligned}
    \mathbf{h}_{u}^{\left ( m,l+1 \right ) }=||_{k=1}^{K} \sigma  \left ( \sum_{v\in N\left ( u \right ) }W_{\left ( uv;m \right ) }^{k } \mathbf{z}_{v}^{l }   \right ), 
\end{aligned}
\end{equation}
\begin{equation}
\begin{aligned}
    \alpha _{u}^{\left ( m,l+1 \right ) }= \sigma  \left (\frac{1}{K} \sum_{k=1}^{K}  \sum_{v\in N\left ( u \right ) }W_{\left ( uv;m \right ) }^{k }  \right ), 
\end{aligned}
\end{equation}
where K is the head of multi-head attention, $\sigma \left ( \cdot  \right ) $ is an activation function, and $\mathbf{z}_{v}^{l}$ is the embedding of the neighbor node. 

In particular, we concatenate the representations of nodes under each semantic without losing too much semantic dependency. 
The representation of M meta-paths in the graph is
\begin{equation}
\begin{aligned}
   \mathbf{z}_{u}^{m} =||_{m=1}^{M} \mathbf{h}_{u}^{m}. 
\end{aligned}
\end{equation}
Since nodes in heterogeneous data are more vulnerable to semantic inference attacks, we further consider the impact of semantic-level on node representation. 
The semantic attention from MLP as
\begin{equation}\label{eq:semanAttenWeight}
\begin{aligned}
   \mathcal{W} _{m} =\frac{1}{N} \sum_{u\in N}\textsc{LeakyReLU}  \left ( w_{2}\mathbf{z}_{u}^{m}+b \right ), 
\end{aligned}
\end{equation}
\begin{equation}\label{eq:semanAttenSoftmax}
\begin{aligned}
   \beta _{m} =\textsc{Softmax} \left ( \mathcal{W}_{m}   \right ) , 
\end{aligned}
\end{equation}
where $\mathcal{W} _{m}$ is the attention weight of $m$, $\beta _{m}$ is the normalized attention coefficient, and each node in $gs_m$ shares an attention coefficient. 
So we get the multi-level embeddings in the feature representation learning as
\begin{equation}
\begin{aligned}
   \mathbf{z} _{u} = \sum_{m}^{M} \beta _m \mathbf{z} _{u}^{m} . 
\end{aligned}
\end{equation}

Subsequently, we inject noise uniformly into the nodes individually, which means our noise fuse the weights of neighbor and semantic. 
We design the sensitivity and Gaussian noise on heterogeneous graph following Eq.~\eqref{eq:defineSens} as
\begin{equation}\label{eq:nodeNoise}
\begin{aligned}
    &\triangle _{2}\mathcal{S}_{feat} =\max_{D,D^{\prime } } \alpha _{u}^{m} \beta _{m}\cdot  \left \| \mathcal{S} \left ( D \right ) - \mathcal{S} \left ( D^{\prime }  \right )   \right \|_{2}, \\
    &\widetilde{\mathbf{h}} = \mathbf{z} _{u}+\lambda \cdot \mathcal{N}_{feat}^{u} \left ( 0, \sigma _{\epsilon_f}^{2} (\triangle _{2}\mathcal{S}_{feat})^2 \mathbf{I}  \right ) , 
\end{aligned}
\end{equation}
where $\lambda$ is a hyperparameter, the privacy budget $\epsilon _{f} < \epsilon $ and $\mathcal{N}_{feat}^{u}$ is the Gaussian distribution with mean $0$ and standard deviation $ \sigma _{\epsilon_f} \triangle _{2}\mathcal{S}_{feat} $ for $u$ to satisfy $\left ( \epsilon _{f},\delta   \right ) $-DP. 
\begin{algorithm}[!t]
\LinesNumbered
\label{alg}
\caption{HeteDP. 
} 
\KwIn{Heterogeneous Graph $G$; 
Negative sampling $k$; 
Local privacy budget $\epsilon_f $ or $\epsilon_s $; 
Node feature $\mathbf{h}$; 
Meta-path $m$; 
Multi-head attention $K$; 
Number of training epochs $T$; 
Batch size $B$; 
Noise scale $\sigma $; 
Gradient norm bound $C$
.}
\KwOut{Predicted result of the downstream task. 
}
Initialize all model parameters; \\
\tcp{Feature Learning}
Generate semantic subgraphs $gs$ from $G$ with $m$; \\
Calculate node attention $\alpha^{K} $ by Eq.~\eqref{eq:nodeAttenScore}, Eq.~\eqref{eq:nodeAttenCof}; \\
Calculate semantic attention $\beta _m$ by Eq.~\eqref{eq:semanAttenWeight}, Eq.~\eqref{eq:semanAttenSoftmax}; \\
Add noise $\mathcal{N}_{feat}^{\epsilon_f}$ to $\mathbf{h}$ by Eq.~\eqref{eq:nodeNoise}; \\
Get perturbed nodes $\widetilde{\mathbf{h}}$; \\
\tcp{Topology Learning}
Sample negative nodes $\mathbf{h}_k^{\prime } $ by Eq.~\eqref{eq:negSample}; \\
\For{$t = 1, 2, \cdots, T$}{
    Calculate node embeddings and reparameterization $q\left ( \cdot  \right ) $ by Eq.~\eqref{eq:encoder}; \\
    Reconstruct edges by Eq.~\eqref{eq:decoder}; \\
    Calculate loss between $\mathbf{h}$ and $\mathbf{h}_k^{\prime } $ by Eq.~\eqref{eq:HVGAEloss}; \\
    Get gradient $\mathbf{g}$ and update gradient $\widetilde{\mathbf{g}}$ with injecting noise $\mathcal{N}_{topo}^{\epsilon _{s},B,\sigma ,C}$ by Eq.~\eqref{eq:topNoise}; \\
}
\end{algorithm}

\subsection{Topology Learning with Privacy-Preserving}
We design a feature encoder for heterogeneous graphs and a topology reconstructor to execute the heterogeneous differentially private stochastic gradient descent to achieve privacy-preserving on topological structures. 

\textbf{Feature encoder. }
Inspired by RGCN~\cite{RGCN2018}, we build a simple RGCN model to extend VGAE~\cite{VGAE2016} to handle heterogeneous data. 
All of these aggregate representations of feature and relationship form a heterogeneous GCN model and the hidden layer in the model is
\begin{equation}\label{eq:HeteGCN}
\begin{aligned}
    &\mathbf{\widetilde{h} }_{dst}^{\left ( l+1 \right ) } =\textsc{Agg} \left ( f_{r} \left ( G,\mathbf{\widetilde{h} }_{src}^{l},\mathbf{\widetilde{h} }_{dst}^{l}  \right )|_{r\in R}   \right ) \\
    &\mathbf{s.t.}~\mathbf{\widetilde{h} }=\textsc{HeteGCN}(\mathbf{X},\mathbf{A} _r ), 
\end{aligned}
\end{equation}
where $f_{r}$ is the GCN module of each ${r\in R}$, $\mathbf{X}$ is node features, and $\mathbf{A} _r $ is the relationship matrix. 
The hidden layer representation of each node under the relational subgraph as
\begin{equation}\label{eq:HeteGCNLayer}
\begin{aligned}
    \mathbf{\widetilde{h} }_{u}^{\left ( l+1 \right ) }= \sigma \left ( \sum_{v\in N_{(u)} }\zeta   w^{l} \mathbf{\widetilde{h} }_{v}^{l}   \right ),
\end{aligned}
\end{equation}
where $\zeta  $ is a normalization constant, $w^{l}$ and $\mathbf{h}_{v}^{l}$ are the learnable weight matrices and neighbor node embeddings of the $l$-th layer. 

Since we transform the recommendation task on heterogeneou graph neural networks into a graph reconstruction problem, we follow the original intention of the link reconstruction task. 
We train a link prediction model by computing the difference in scores between two connected nodes and any pair of nodes. 
For example, there is an edge $e\in E$ between nodes $u\in V$ and $v\in V$ in graph $G$, and we want the score between $u$ and $v$ to be higher than the score between $u$ and $k$ nodes $v^{\prime } $ sampled from an arbitrary distribution $v^{\prime } \sim \mathrm{Pn} \left ( v \right ) $. 
We uniform sample a different sample for each iteration of training through the neighbor sampling of the multi-layer GNN as negative sampling 
\begin{equation}\label{eq:negSample}
\begin{aligned}
    \mathbf{h}_{v^{\prime } }^{k} \gets \textsc{NegSample} \left ( \mathbf{h}_{u} ,k  | \forall u\in V   \right ) . 
\end{aligned}
\end{equation}
Then, we adopt a two-layer HeteGCN model following Eq.~\eqref{eq:HeteGCN} as an encoder and utilize the reparameterization trick in training
\begin{equation}\label{eq:encoder}
\begin{aligned}
    q\left ( \mathbf{Z}|\mathbf{X},\mathbf{A}_{r}  \right ) =\prod_{i=1}^{N} \mathrm{Pn}\left ( \mathbf{z}_{i}|\mu_{i}^{r},\left ( \sigma _{i}^{2}  \right )^r \right ) , 
\end{aligned}
\end{equation}
where $\mathbf{z}$ is a stochastic latent sampling variable, $\mu _{r} =\textsc{HeteGCN}_{\mu } \left ( \mathbf{X},\mathbf{A}_{r} \right ) $ is the matrix of mean vectors $\mu_{i}^{r}$ and $\log_{\sigma }^{r} =\textsc{HeteGCN}_{\sigma } \left ( \mathbf{X},\mathbf{A}_{r}   \right ) $ is the matrix of standard deviation vectors $\sigma _{i}^{r}$. 

We compute the inner product between latent variables as a decoder to reconstruct the edge.  
We leverage the calculation to express the probability that there is a connection between two different types of nodes $\phi (\mathbf{z}_{u})$ and $\phi (\mathbf{z}_{v})$ as
\begin{equation}\label{eq:decoder}
\begin{aligned}
    p\left ( \mathbf{A}_{r}|\mathbf{Z} \right ) =\prod_{i=1}^{\left | A_{u}  \right | } \prod_{j=1}^{\left | A_{v}  \right | }\sigma \left ( \mathbf{z}_{u}^{T} \mathbf{z}_{v}  \right ) , 
\end{aligned}
\end{equation}
where $ \mathbf{z}_{u}^{T}$ represents the transpose of $\mathbf{z}_{u}$. 

Our goal is to enable the model to find patterns in the data during training and to discover some underlying structure.
Therefore, We can discover the interdependence and association of node $u$ and $v$ based on semantic association rules and calculate the score between the node pair with the unsupervised cross-entropy loss of the graph as
\begin{equation}
\begin{aligned}
    \mathcal{L} _{D}^{r} =-\log \sigma \left ( q_{u,v} \right ) -k\cdot \mathbb{E} _{v^{\prime }\sim \mathrm{Pn} \left ( v \right )}\log \left ( \sigma \left ( -p_{u,v^{\prime } } \right )  \right ), 
\end{aligned}
\end{equation}
where $k$ is the number of negative sampling. 
To further reduce the difference between generated samples and real samples, we compute their $\textsc{KL}$ divergence in the loss function as
\begin{equation}\label{eq:HVGAEloss}
\begin{aligned}
    \mathcal{L} =\mathcal{L} _{D}^{r}-\sum_{i\in \left \langle u,v \right \rangle } \textsc{KL} \left ( q_{i} ||p\left (\mathbf{Z}_{i}  \right ) \right ), 
\end{aligned}
\end{equation}
 where $p\left (\mathbf{Z}_{i}  \right )=\prod_{i}\mathrm{Pn} \left ( \mathbf{z}_{i}|0,\mathbf{I} \right ) $ is a Gaussian prior. 
 We take the state when the graph topology is learned as the optimal prediction of the link. 
 
\textbf{Topology reconstruction with privacy-preserving. } 
In this part, we introduce preserving semantic information in graphs by perturbing the gradient of the link prediction task. 
We inject the Gaussian noise to the training gradient, and further denote as $\triangle _{2}\mathcal{S}_{topo}=C$ following Eq.~\eqref{eq:GaussionNoise}. 

Then, for each iteration in training, we calculate the gradient of decoder $\mathbf{g}=\nabla \mathcal{L} $ from backpropagation, inject noise into the gradient after gradient clipping and before gradient update, and finally perform gradient descent. 
Thus, the perturbed gradients as
\begin{equation}\label{eq:topNoise}
\begin{aligned}
   \widetilde{\mathbf{g}} =\frac{1}{\left | B \right | } \left ( \sum_{i\in B} \mathbf{g}_i^R /\max \left ( 1,\frac{\left \| \mathbf{g}_{i}^R  \right \|_{2}  }{C }  \right )+\mathcal{N}_{topo}\left ( 0,\sigma _{\epsilon _{s} }^{2} C^{2}\mathbf{I}   \right ) \right ), 
\end{aligned}
\end{equation}
where $B$ and $|B|$ is the batch and size for each training iteration, $\left \| \mathbf{g}_{i}^R  \right \|_{2} $ is the $l_2$ norm of gradient clipping, and $\mathcal{N} _{topo}\left ( \cdot  \right ) $ is the Gaussian distribution with mean $0$ and standard deviation $\sigma _{\epsilon _{s} }C$. 
The distribution satisfies $\left ( \epsilon _{s},\delta   \right ) $-DP, where the privacy budget $\epsilon _{s}< \epsilon$. 
We control the sensitivity to noise by limiting the norm bound $C$ of a gradient. 
To adapt to the noise distribution in heterogeneous data, we utilize privacy accounting~\cite{DPSGD2016} to regulate the privacy budget of each iteration. 
We set a constant number $c_2$, the sampling probability $P$, and the number of iterations $T$ for training to make $\sigma \epsilon _{s} \ge c_{2} P\sqrt{T\log{1/\delta } } $. 

\subsection{Privacy-Preserving Analysis of HeteDP}
In this section, we give a general overview of the proposed privacy-preserving framework for heterogeneous graphs.
And then, we perform privacy analysis and proof. 
We aim to leverage a bi-level optimization strategy to maximize the privacy-preserving effect while minimizing the information loss due to noisy inputs. 
In general, our proposed scheme solves the problem of privacy budget allocation in both feature noise and topology noise as Eq.~\eqref{eq:problem}.
The whole algorithm process is elaborated in Algorithm~\ref{alg}. 
The features used in topology learning come from node representation learning~\ref{subsec:featlearning} as designed above. 
Moreover, we prove the privacy of HeteDP in the following theorem. 
\newtheorem{thm}{\bf Theorem}
\begin{thm}
A random function $\mathcal{M}$ is $\left ( \epsilon,\delta   \right ) $-DP if the privacy loss $\mathcal{C} _{\mathcal{M} } \left ( o,D,{D}'  \right ) $ satisfies $\mathrm{Pr} \left [ \mathcal{C} _{\mathcal{M} } \ge \epsilon  \right ] \le \delta $, where the privacy loss define as
\begin{equation}
\begin{aligned}
  \mathcal{C} _{\mathcal{M} } \left ( o,D,{D}'  \right ): =\ln{\frac{\mathrm{Pr} \left [ \mathcal{M}\left ( D \right ) =o \right ] }{\mathrm{Pr} \left [ \mathcal{M}\left ( {D}'  \right ) =o \right ] } }. 
\end{aligned}
\nonumber
\end{equation}
\end{thm}
\begin{proof}
Let us partition $\mathbb{O} $ as $\mathbb{O} =\mathcal{O} \cup{\mathcal{ O} }' $, where $\mathcal{O} =\left \{ o\in \mathbb{O} :\mathcal{C _{M} }\ge \epsilon _{f,s}  \right \} $ and ${\mathcal{ O} }' =\left \{ o\in \mathbb{O} :\mathcal{C _{M} }< \epsilon _{f,s}  \right \} $. 
For any $S\subseteq \mathbb{O} $, if $\mathrm{Pr} \left [\mathcal{C} _{\mathcal{M} } \left ( o,D,{D}'  \right )\ge \epsilon_{f,s} \right ]\le \delta $, we have
\begin{equation}
\begin{aligned}
  &\mathrm{Pr} \left [ \mathcal{M}\left ( D \right )\in S \right ]\\ &=\mathrm{Pr} \left [ \mathcal{M}\left ( D \right )\in S \cap \mathcal{O}  \right ]
  +\mathrm{Pr} \left [ \mathcal{M}\left ( D \right )\in S\cap {\mathcal{O}}'  \right ]\\
  &\le \mathrm{Pr} \left [ \mathcal{M}\left ( D \right )\in \mathcal{O}  \right ]+\exp\left (  \epsilon _{f,s} \right )  \mathrm{Pr} \left [ \mathcal{M}\left ( {D}'  \right )\in S\cap {\mathcal{O}}'  \right ]\\
  &\le \delta +\exp\left (  \epsilon_{f,s} \right )  \mathrm{Pr} \left [ \mathcal{M}\left ( {D}'  \right )\in S \right ],
\end{aligned}
\nonumber
\end{equation}
yielding $\left ( \epsilon,\delta   \right ) $-DP for the Gaussian mechanism, where $\epsilon _{f,s}$ denotes the privacy budget of noise on node features or topology. 
\end{proof}
\begin{thm}
Let $\mathcal{M} _1:D\to \mathcal{O} _1$ be an $\left ( \epsilon _f,\delta  \right ) $-DP algorithm, and $\mathcal{M} _2:D\to \mathcal{O} _2$ be an $\left ( \epsilon _s,\delta  \right ) $-DP algorithm. 
Their combination defined to be $\mathcal{A}=\mathcal{M} _{1,2}:D\to \mathcal{O} _1\times \mathcal{O} _2$ by the mapping: $\mathcal{A} \left ( x \right ) =\left ( \mathcal{M}_1\left ( x \right ) ,\mathcal{M}_2\left ( x \right )   \right ) $ is $\left ( \epsilon _f+\epsilon _s,\delta  \right ) $-DP. 
\end{thm}
\begin{proof}
Let $x,y\in D$ and fix $\forall \left ( o _1,o _2 \right ) \in \mathcal{O} _1\times\mathcal{O} _2$. 
Then
\begin{equation}
\begin{aligned}
&\mathrm{Pr} \left [ \mathcal{A}\left ( D \right ) =\mathcal{O}   \right ] +\delta \\
&=\frac{\left ( \mathrm{Pr} \left [ \mathcal{M}_1\left ( x \right ) =o_1 \right ] +\delta \right ) \left ( \mathrm{Pr} \left [ \mathcal{M}_2\left ( x \right ) =o_2 \right ] +\delta \right )  }{\left ( \mathrm{Pr} \left [ \mathcal{M}_1\left ( y \right ) =o_1 \right ] +\delta \right )\left ( \mathrm{Pr} \left [ \mathcal{M}_2\left ( y \right ) =o_2 \right ] +\delta \right )  }\\
&=\left ( \frac{\mathrm{Pr} \left [ \mathcal{M}_1\left ( x \right ) =o_1 \right ] +\delta}{\mathrm{Pr} \left [ \mathcal{M}_1\left ( y \right ) =o_1 \right ] +\delta } \right )\left ( \frac{\mathrm{Pr} \left [ \mathcal{M}_2\left ( x \right ) =o_2 \right ] +\delta}{\mathrm{Pr} \left [ \mathcal{M}_2\left ( y \right ) =o_2 \right ] +\delta } \right )\\
&\le \exp \left ( \epsilon _f \right ) \exp \left ( \epsilon _s \right ) =\exp \left ( \epsilon _f +\epsilon _s\right ) ,
\end{aligned}
\nonumber
\end{equation}
which shows that the combination algorithm $\mathcal{A}$ satisfies $\left ( \epsilon _f+\epsilon _s,\delta  \right ) $-DP. 
\end{proof}

\begin{table*}[t]
\caption{Summary of experimental results: “F1 score in NC and ROC-AUC score in LP” (\%). 
\vspace{-0.5em}
}
\centering 
\resizebox{\textwidth}{!}{
\setlength{\tabcolsep}{5mm}{
\begin{tabular}{lccccccccccc}
\toprule
\textbf{Dataset} & \multicolumn{2}{c}{\textbf{ACM}} & \multicolumn{2}{c}{\textbf{DBLP}} & \multicolumn{2}{c}{\textbf{Amazon}} &
\multicolumn{2}{c}{\textbf{IMDB}} \\ 
\textbf{Task} & NC & LP & NC & LP & NC & LP & NC & LP\\
\midrule
HGConv~\cite{HGCN2020}                   & \textbf{88.89} & 82.15 & 93.40 & 57.00 & 92.17 & 63.43 & 63.43 & 64.00\\
HGT~\cite{HGT2020}                       & 88.85 & 79.68 & \textbf{93.42} & 53.32 & 94.40 & 65.77 & \textbf{63.63} & 55.52\\
Metapath2vec~\cite{metapath2vec2017}     & 73.69 & $-$     & 92.80 & 44.58 & 78.33 & \textbf{88.86} & 48.81 & 67.69\\
RGCN~\cite{RGCN2018}                     & 82.67 & 63.28 & 87.70 & 58.10 & 94.50 & 65.03 & 61.30 & 74.32\\
HetGNN~\cite{HetGNN2019}                 & 82.82 & \textbf{89.99} & 90.43 & 55.73 & 70.61 & 72.37 & 54.78 & 59.24\\
\midrule
HeteDP (no)                           & 87.50 & 85.44 & 87.33 & \textbf{79.72} & \textbf{97.82} & 72.52 & 53.07 & \textbf{82.07}\\
HeteDP ($\epsilon$=0.01)                 & \underline{67.33} & \underline{71.92} & \underline{30.13} & \underline{62.83} & \underline{95.28} & \underline{60.20} & \underline{40.12} & \underline{74.37}\\
HeteDP ($\epsilon$=0.1)                  & 76.15 & 72.15 & 32.24 & 68.84 & 97.41 & 64.65 & 40.29 & 75.06\\
HeteDP ($\epsilon$=1)                    & 80.33 & 77.39 & 39.81 & 73.94 & 97.67 & 72.24 & 48.50 & 75.57\\
\bottomrule
\end{tabular}}
}
\vspace{-2em}
\label{summary results}
\end{table*}

\section{Experiments}
\subsection{Experimental Setup}
\begin{figure*}[!t]
\centering
\subfigure[LP on ACM.]{
		\includegraphics[width=0.30\textwidth, height=0.20\textwidth]{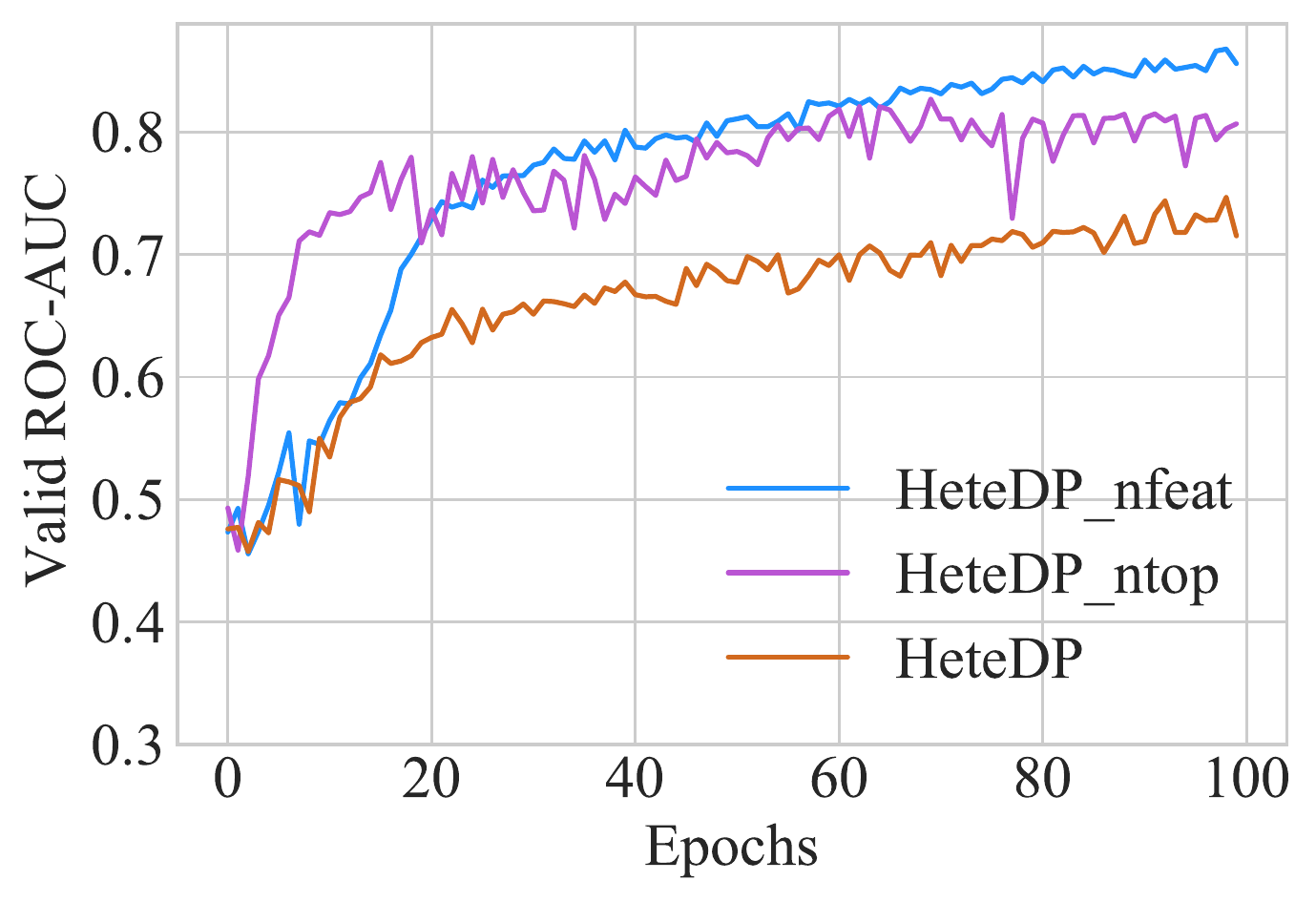}
	}
\hspace{0.5em}
\subfigure[LP on Amazon.]{
		\includegraphics[width=0.30\textwidth, height=0.20\textwidth]{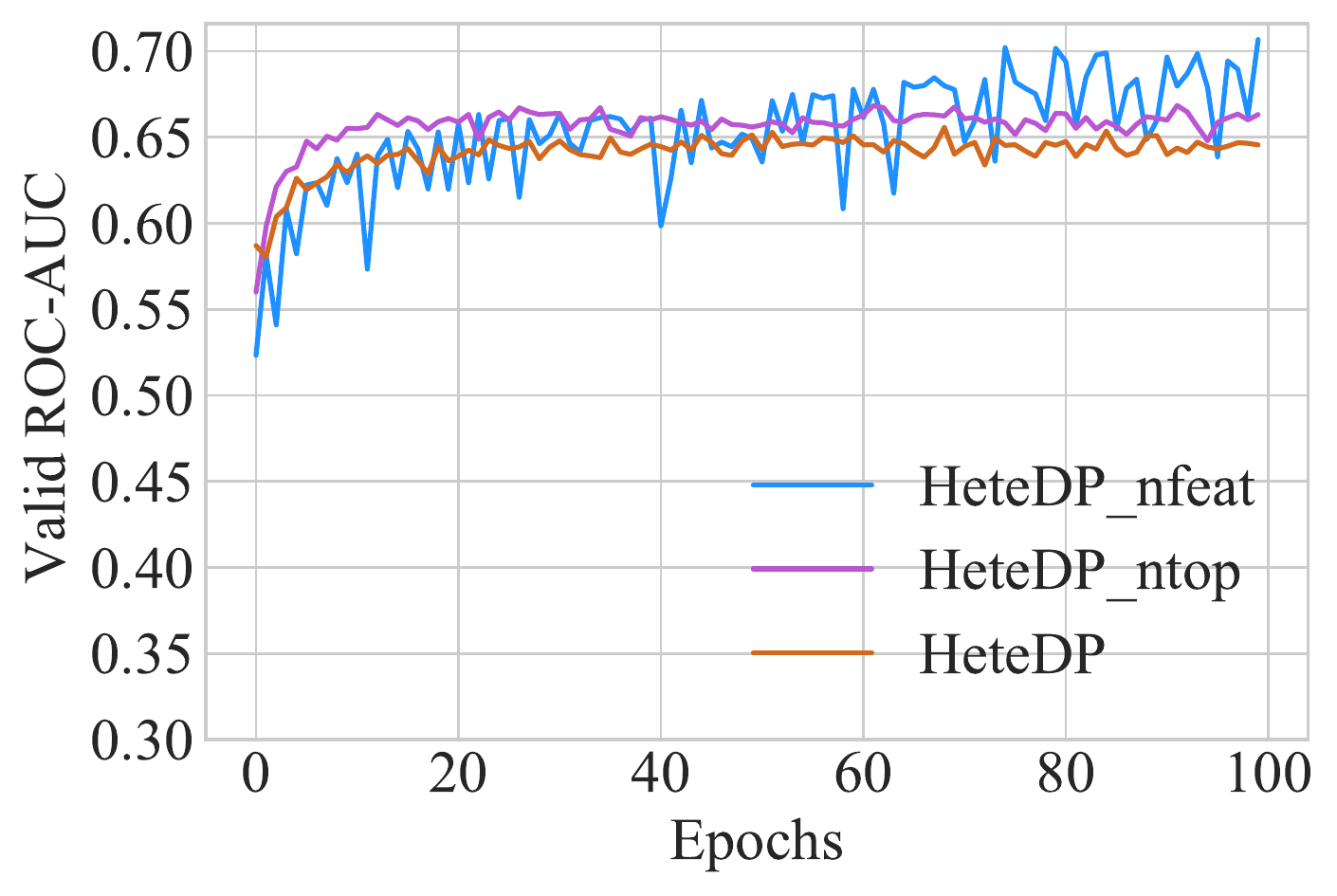}
	}
\hspace{0.5em}
\subfigure[LP on IMDB.]{
		\includegraphics[width=0.30\textwidth, height=0.20\textwidth]{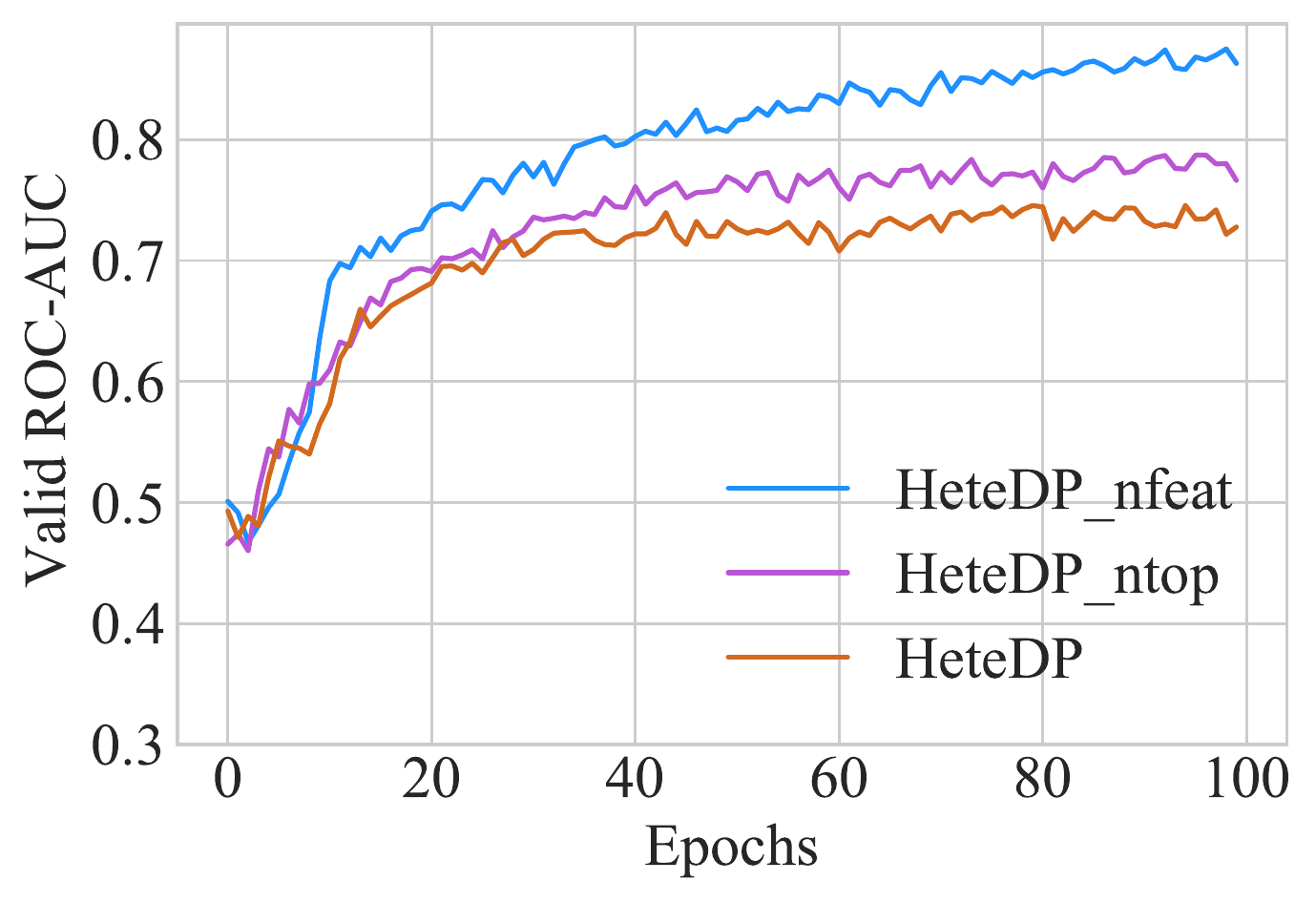}
	}
\vspace{-0.5em}
\caption{Ablation study of ROC-AUC scores of LP on validation set with $\epsilon=0.01$. }
\vspace{-2em}
\label{fig:ablation}
\end{figure*}

\begin{figure*}[!t]
\vspace{1em}
\centering
\subfigure[Original model.]{
	\includegraphics[width=0.21\textwidth, trim=10 10 10 10, clip]{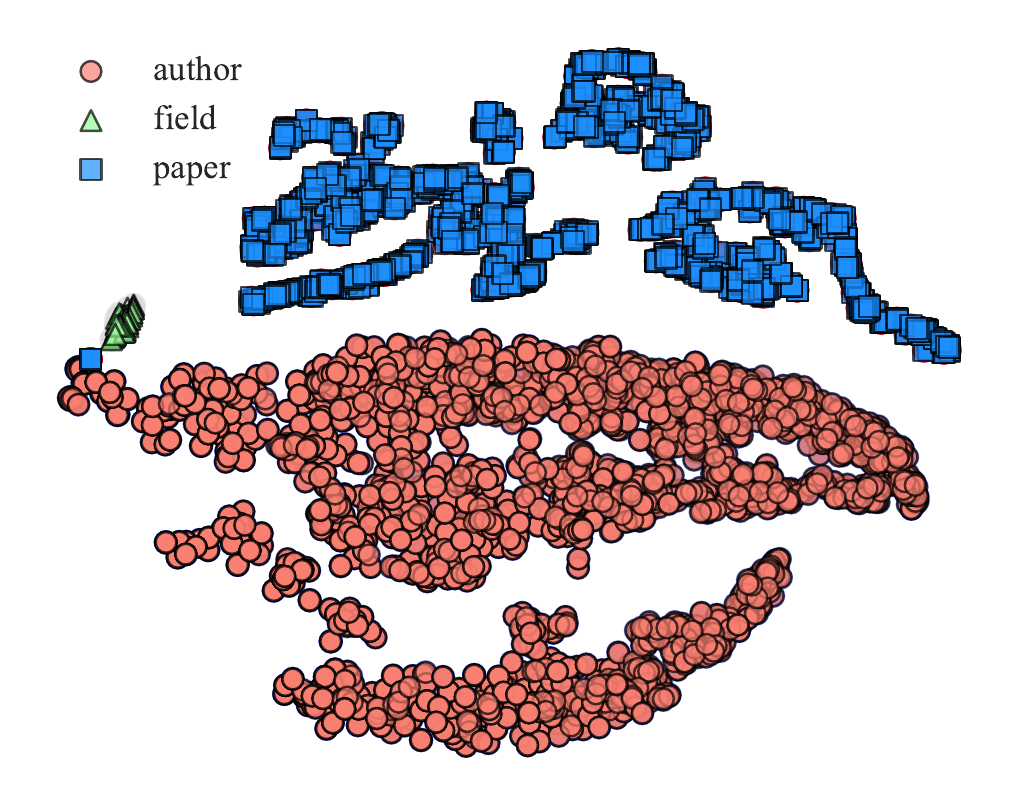}
}
\hspace{0.5em}
\subfigure[Feature perturbed.]{
	\includegraphics[width=0.21\textwidth, trim=10 10 10 10, clip]{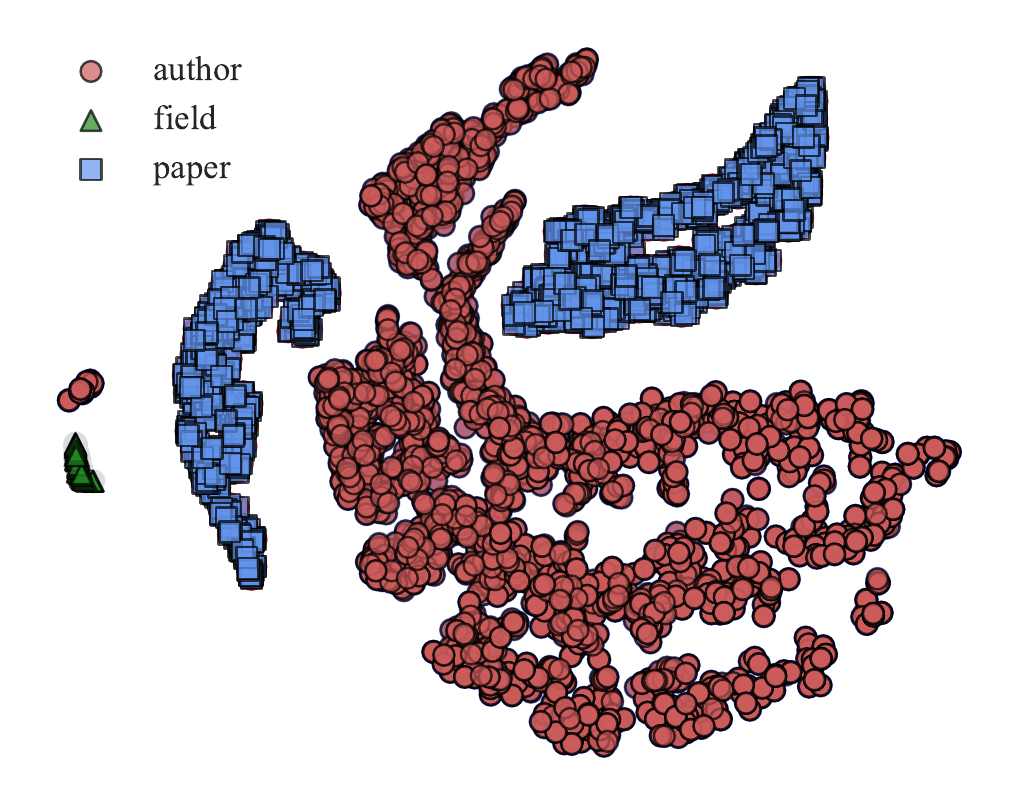}
}
\hspace{0.5em}
\subfigure[Topology perturbed.]{
	\includegraphics[width=0.21\textwidth, trim=10 10 10 10, clip]{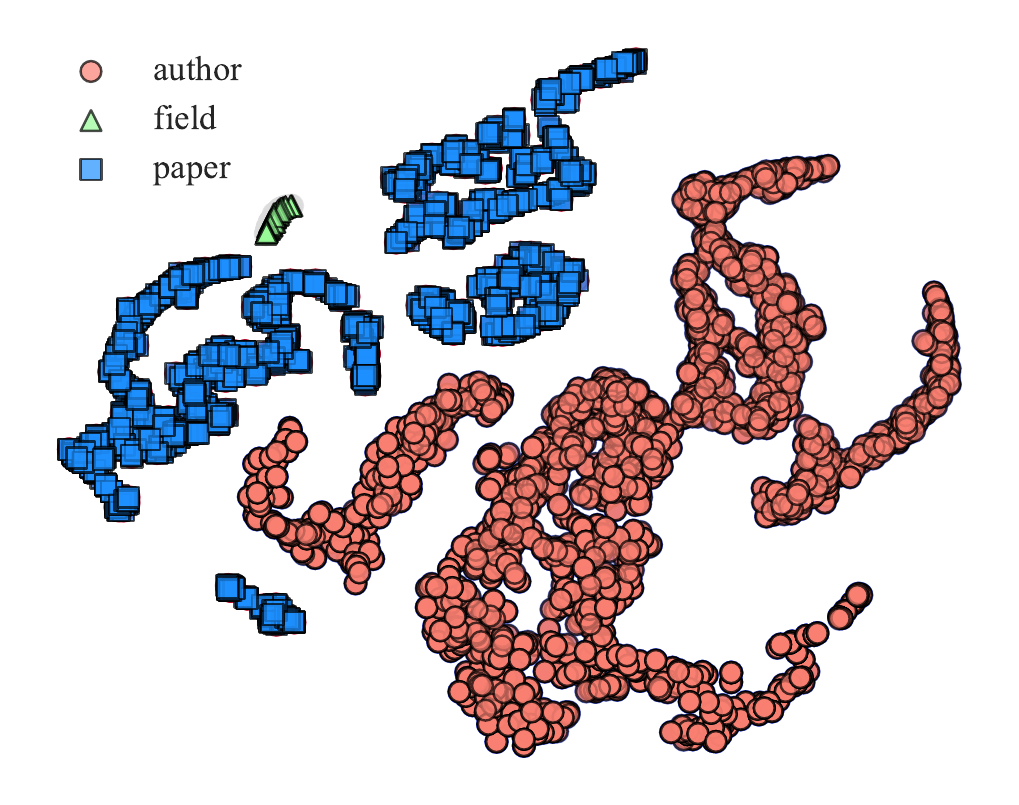}
}
\hspace{0.5em}
\subfigure[HeteDP.]{
	\includegraphics[width=0.21\textwidth, trim=10 10 10 10, clip]{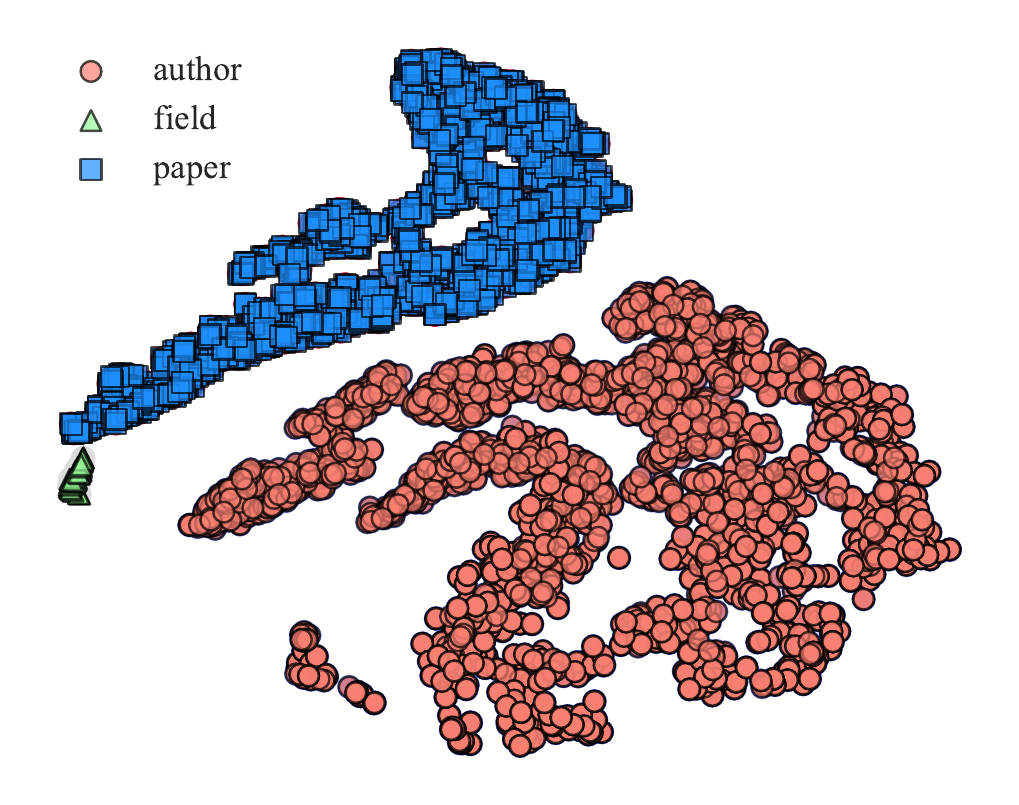}
}
\vspace{-0.5em}
\caption{The visualisation of node types on ACM. 
}
\vspace{-2em}
\label{fig:ncVisualisation}
\end{figure*}
In this section, we conduct experiments in all aspects on four datasets and two tasks to demonstrate the adaptability of heterogeneity privacy-preserving and the effectiveness of graph learning. 
The experiment results of HeteDP are shown in Table~\ref{summary results}, where the best accuracy shown in bold and the best privacy-preserving results are are underlined. 
Furthermore, ``$-$" indicates that the current model hardly implement in the dataset. 
We then further analyze how HeteDP is affected by changing the strength of privacy-preserving, and our contribution to the overall performance of the optimization model.

\begin{table}[!t]
\caption{Statistics of Datasets. }
\vspace{-0.5em}
\centering
\begin{tabular}{lrrrr}
\toprule
\textbf{Dataset} & \textbf{\# Nodes} & \textbf{\# Edges} \\
\midrule
\multirow{3}*{ACM}      & author: 17,351   &  \textbf{paper-author}: 13,407  \\
                       ~& \textbf{paper}: 4,025     & paper-field: 4,025    \\
                       ~& field: 72        &                      \\
\midrule
\multirow{4}*{DBLP}     &  \textbf{author}: 4,025    &  \textbf{paper-author}: 19,645   \\
                       ~& paper: 14,328    & paper-conf: 14,328   \\
                       ~& conf: 20         & paper-term: 85,810   \\
                       ~& term: 7,723      &    \\
\midrule
\multirow{4}*{Amazon}   & user: 6,170      &  \textbf{user-item}: 195,791   \\
                       ~&  \textbf{item}: 2,753      & item-view: 5,694   \\
                       ~& view: 3,857      & item-category: 5,508   \\
                       ~& category: 22     &    \\
\midrule
\multirow{3}*{IMDB}     &  \textbf{movie}: 4,278     &  \textbf{movie-actor}: 12,828  \\
                       ~& actor: 5,257     & movie-director: 4,278    \\
                       ~& director: 2,081  &                      \\
\bottomrule
\end{tabular}
\vspace{-2em}
\label{banchmark dataset}
\end{table}
\textbf{Datasets. } 
We use four open datasets, including citation networks (ACM and DBLP), an E-commerce dataset (Amazon), and a relational movie network (IMDB). 
The dataset statistics are shown in Table~\ref{banchmark dataset}. 
We mark the classified nodes and the predicted edges with bolded. 
For example, in the downstream task of the ACM dataset, we perform node classification for ``paper" and link prediction for ``paper-author". 

\textbf{Baselines. } 
We compare the HeteDP with state-of-the-art heterogeneous baseline methods, including HGConv~\cite{HGCN2020}, HGT~\cite{HGT2020}, metapath2vec~\cite{metapath2vec2017}, RGCN~\cite{RGCN2018}, and HetGNN~\cite{HetGNN2019}. 
In addition, based on these models, we extend the corresponding unsupervised link prediction task, and they omit privacy-preserving. 
The details of these methods are as follows: 
\begin{itemize}[leftmargin=*]
    \item \textbf{HGConv}~\cite{HGCN2020} introduces node representation based on mixed micro/macro level convolution operations on heterogeneous graphs. 
    A micro-level convolution can learn the dependency of nodes under the constraints of the same relation, and a macro-level convolution is used to distinguish subtle differences between relation types. 
    \item \textbf{HGT}~\cite{HGT2020} introduces a attention mechanisms to learn the correlation of node-type and edge-type, which can capture dynamic information about the network, avoid formulate meta-paths, and be better used on large-scale graphs. 
    \item \textbf{Metapath2vec}~\cite{metapath2vec2017} employs the meta-paths-based random walk on the skip-gram method  to reserve semantic information for heterogeneous graph embedding. 
    \item \textbf{RGCN}~\cite{RGCN2018} uses parameter sharing and sparse constraint techniques, applied to handle large amounts of multi-relational data, and has significant results in entity classification. 
    \item \textbf{HetGNN}~\cite{HetGNN2019} considers each node's heterogeneous content (node's attribute information) and uses random walk to sample a fixed number of strongly associated heterogeneous neighbors for graph nodes, and then uses BiLSTMs to process the heterogeneous information. 
    Since the concatenated edges between nodes of the same type are not included in our dataset, we ignore the fusion of node representations in this part. 
\end{itemize}

\textbf{Settings. } 
We set the parameters of feature learning and topology learning separately, with learning rates $lr$ of $0.005$ and $0.001$, epsilon $\epsilon$ from $0.01$ to $1$ and $0.01$ to $0.5$, hidden layer dimension of $64$ and $32$. 
The common parameters set epoch to $100$ and the probabilistic of breaking privacy-preserving $\delta$ to $1e-5$. 
The categories of node classification and edge prediction set for each data select follow Table~\ref{banchmark dataset}. 
We count the meta-paths $m$ for each node type to include all links as much as possible. 
The number of layers depends on the meta-paths $m$ and the types of edges $R$ in the graph. 
In addition, we set the unique parameters in the first part of learning as the dropout of training to $0.8$, the regularization coefficients to $0.001$, the number of heads of the multi-headed attention mechanism $K$ to $8$, and a hyperparameter $\lambda$ to $0.01$. 
In topology learning, we set the batch size $\left | B \right | $ to $2048$ and the number of negative sampling $k$ to $5$. 
We follow the dataset split setting in~\cite{VGAE2016}. 
For the baseline models, the parameters are set as the default values in their papers.

\subsection{Performance Comparison}
We set up two downstream tasks to test the performance of our proposed method, node classification (NC) and link prediction (LP). 
Table~\ref{summary results} summarizes the performance of HeteDP in different downstream task and on four datasets, comparing with the baseline methods, which reflects the inherent generalizability of the HeteDP and the effect when it has differential privacy-preserving. 

\textit{For the node classification task}, we consider the practice of unsupervised node classification~\cite{graphsage2017}, using negative sampling of edges for training and $2$-order neighbor sampling at each iteration of validation. 
We use the F1 score as a classification effectiveness measure. 
The experimental results show that HeteDP reduces the F1 score of node classification by at least 12.95\% on IMDB. 
\textit{For the link prediction task}, we extend the sampler~\cite{graphsage2017} to negative sampling on heterogeneous graphs, sampling $k$ negative pairs for each edge. 
Each training randomly selects a specific size of data to form batch training. 
The encoder consists of heterogeneous convolutional layers Eq.~\eqref{eq:HeteGCN}, Eq.~\eqref{eq:HeteGCNLayer}, and the decoder calculates the scores of positive and negative sample pairs by inner product, respectively. 
So the probability of successful link prediction is a mapping of the recommended probability. 
We utilize the ROC-AUC score as an indicator to judge the performance of HeteDP. 
In terms of ROC-AUC score, HeteDP also reduces by 13.52\% and 7.7\% on ACM and IMDB. 

Overall, in the LP task of DBLP and IMDB, compare to the second-best model, our proposed original model improves performance over $21.62\%$ and $7.75\% $.
The noise of different sensitivities to each dataset brings diverse levels of influence. 
Still, in general, the model accuracy improves in different magnitudes with an increasing privacy budget, such as the ROC-AUC score of Amazon is only reduced by $0.28\%$ with $\epsilon=1$.
It shows that the generalization ability of our model is guaranteed to a certain extent, and the model can maintain the utility of the data under the influence of noise.
Similar to what was elaborated above, the ACM dataset has an accuracy reduction of about $14\%$ on the LP task when setting the privacy-preserving strength of $\epsilon=0.01$.
It shows that our proposed privacy-preserving method is resistant to graph topological inference attacks to a certain extent. 

\subsection{Further Analysis}
\textbf{Ablation study}. We further conduct ablation experiments to assess the importance degree of two parts of our proposed. 
We design a total of three experiments for the LP task for comparison: 
the first is to only protect the features of various node types by Eq.~\eqref{eq:nodeNoise} in feature learning and then use the learned node representation with noise for topology learning; 
second, the representation of features aggregates the information of node neighbors and semantics to use for topology learning, while the link relationship is protected during the topology learning process with Eq.~\eqref{eq:topNoise}; 
the final set is the node feature and topology data are double-protected. 
Their privacy budget is $0.01$ and results are shown in \figurename~\ref{fig:ablation}. 
From the experimental results, we can observe that the perturbation of each part is effective.
Compared with the node feature disturbance, the disturbance to the topology structure affects the link prediction accuracy more.
Nevertheless, our training eventually reaches convergence and maintains some utility. 

Furthermore, we visualize node types to observe the utility of privacy-preserving in \figurename~\ref{fig:ncVisualisation}. 
It shows the embedding visualization of all nodes in ACM using t-SNE~\cite{van2008visualizing}, where the different colors indicate node types. 
We design original model, feature perturbed, topology structure perturbed and HeteDP experiments, where the privacy budget is $3$. 
The visualization from left to right generally shows increasingly tight clustering among similar nodes. 
For feature perturbation, we observe that the spacing within the ``paper" node class becomes smaller, which affects the classification effect within that node. 
The boundary between the three-node types is always clear. 
For topological perturbations, it causes a large change in the position of individual nodes even at higher $\epsilon$.
This perturbation phase has a lower impact within the node class, while the different node types become more dispersed.
Finally, compare with the original model, our model reflects the ability to reduce the gap between different types of nodes to some extent compared, which exactly proves the effectiveness of node feature protection and maintains the utility of the HeteDP. 

\begin{figure}[t]
	\centering
	\includegraphics[width=0.48\textwidth, height=0.28\textwidth]{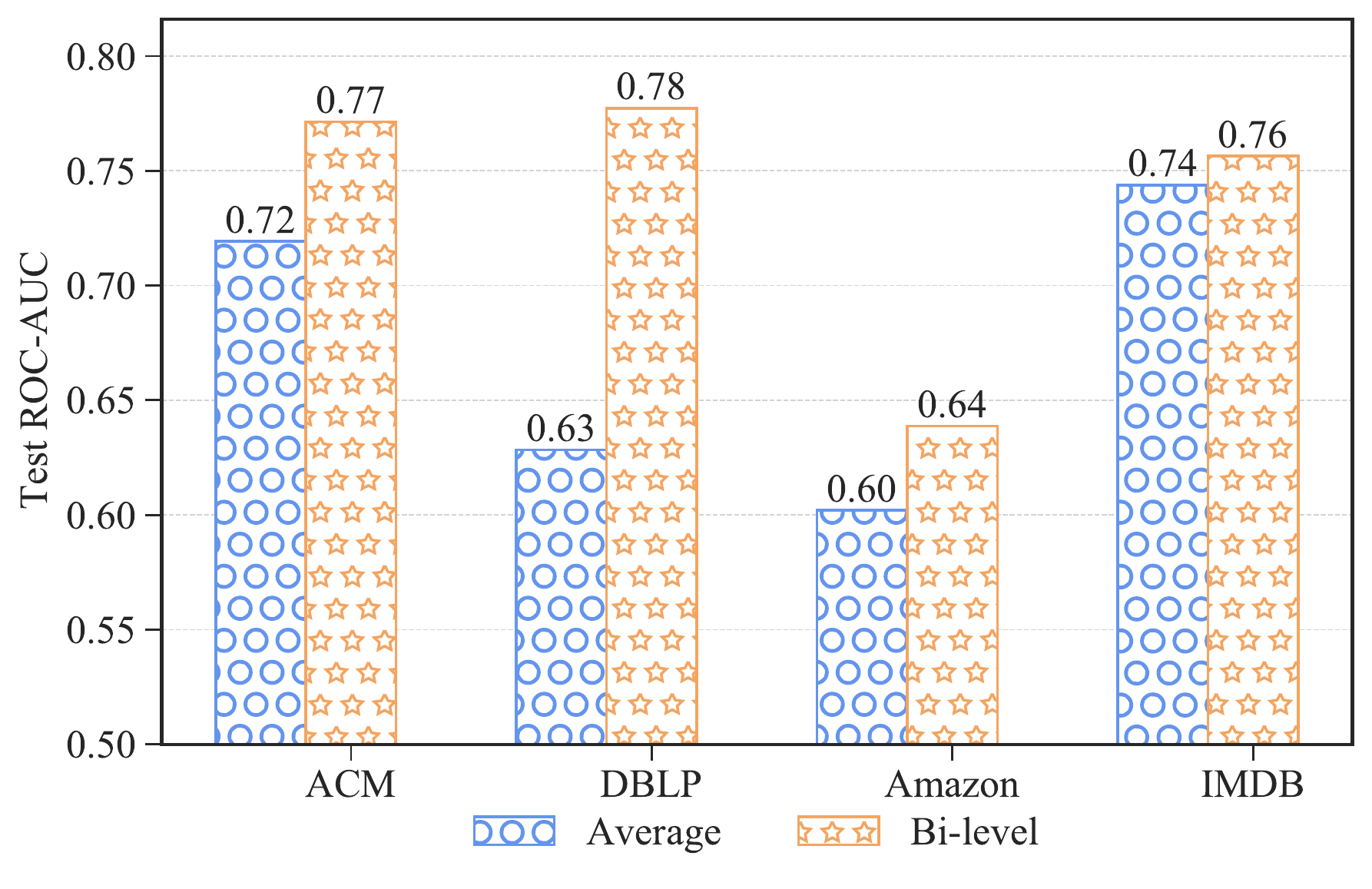}
    \vspace{-1em}
	\caption{Bi-level experiments of ROC-AUC scores of LP on different privacy budget allocation options. 
}
\vspace{-2em}
\label{fig:bilevel}
\end{figure}

\textbf{Bi-level optimization}. Privacy budget allocation has long been an essential task in privacy-preserving. 
The task aims to reduce the probability of data being accessed by attackers, weigh the training accuracy, and consider the problem of model performance degradation due to privacy noise. 
To further improve the utility of the model in privacy-preserving, we design a bi-level optimization trick to allocate the privacy budget of Gaussian noise on the feature part and the topological part. 
We fix the topology noise unchanged and seek the optimal privacy budget allocation on node features by experiments in a specific interval according to Eq.~\eqref{eq:problem}. 
Next, we fix the amount of noise on features to find the optimal privacy budget on topology. 
The results are shown in \figurename~\ref{fig:bilevel}. 
The figure compares the effect of an equally divided privacy budget and a bi-level optimized privacy budget and shows that bi-level optimization can bring better performance for the model, which achieves the purpose of the trade-off between protection power and utility. 

\textbf{Sensitivity Analysis}. We analyze the sensitivity of the overall noise of HeteDP. 
Specifically, we test the extent to which the parameter $\epsilon$ influences our model on the LP task. 
We set $9$ values of $\epsilon$ on ACM, IMDB and DBLP, as shown in \figurename~\ref{fig:epsilon}. 
We observe that the ACM dataset achieves a score close to $80\%$ at $\epsilon=1$, which is nearly $8\%$ higher than $\epsilon=0.01$. 
IMDB, however, is not as sensitive to $\epsilon$ because the network structure of this dataset is fragile, and it is harder to improve the learning ability once it is disturbed. 
The experiments show that different datasets have myriads of changes in sensitivities to the privacy budget due to inconsistencies in their own data distributions, and it is necessary to find a suitable noise range to protect the model and maximize its effectiveness.

\begin{figure}
	\centering
	\includegraphics[width=0.48\textwidth, height=0.28\textwidth]{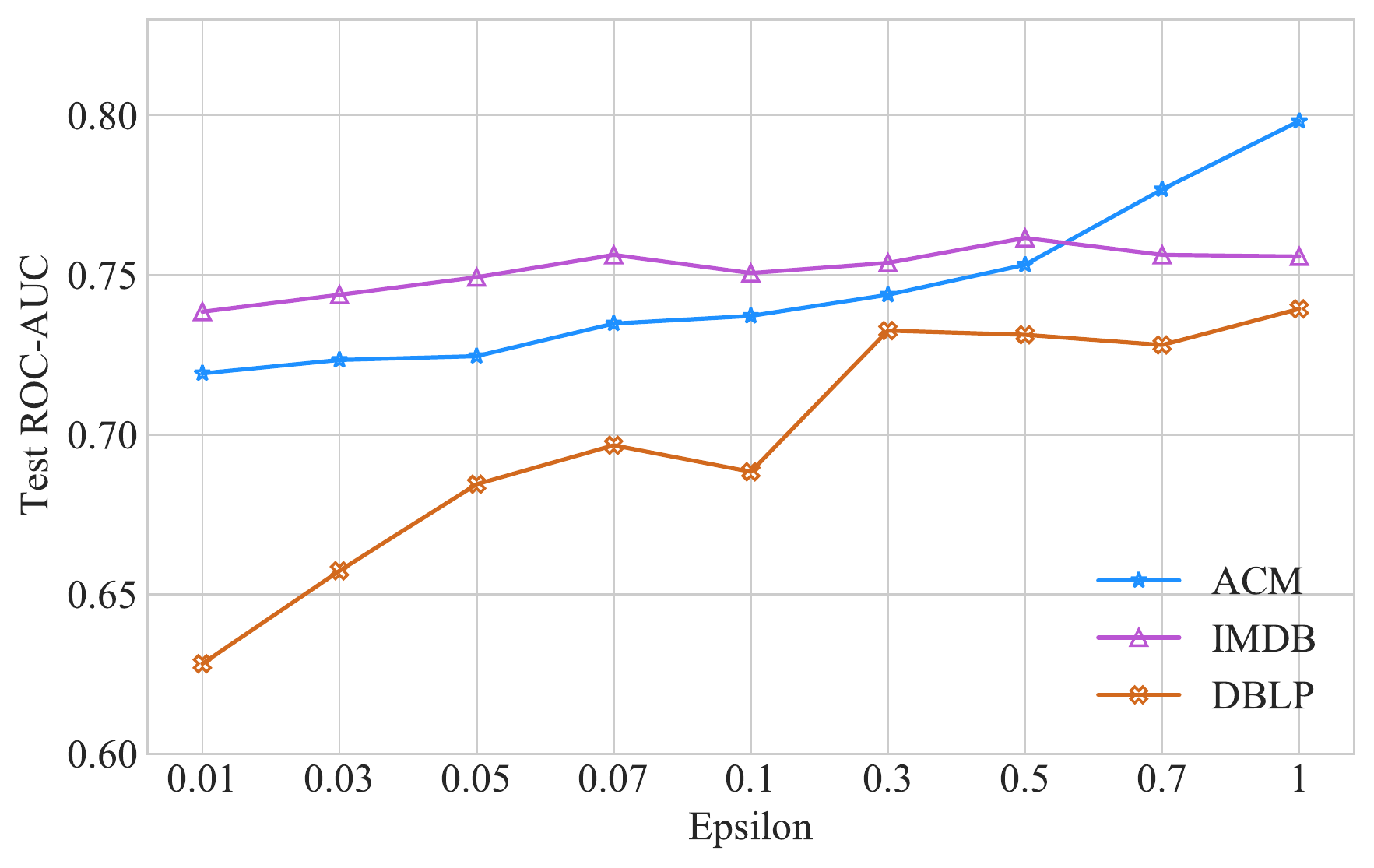}
    \vspace{-1em}
	\caption{Sensitivity experiments on ACM and IMDB. 
}
\vspace{-1em}
\label{fig:epsilon}
\end{figure}

\section{Conclusion}
In this work, we propose HeteDP, a novel privacy-preserving graph learning framework for heterogeneous graph.
We propose a two-stage privacy-preserving mechanism based on differential privacy, capable of adapting to the heterogeneity of heterogeneous graphs. 
For multi-type nodes and relationships on heterogeneous graphs, we learn the representation distribution and aggregation of nodes on each relationship through multi-relational convolutional layers, and adapt to various downstream tasks through unsupervised learning. 
Considering that nodes and links are vulnerable to inference attacks in heterogeneous graph scenarios, we perturb the node features and the topological structure, respectively.
Then, we balance the privacy budget allocation of the node feature and the topology, and achieve the best performance by bi-level optimization. 
Comprehensive experiments on four datasets demonstrate the privacy-preserving capability and adaptability of HeteDP. 
We hope that our work could bring some inspiration to privacy-preserving in more complex graph data.

\section*{Acknowledgment}

This paper was supported by the National Natural Science Foundation of China (Nos. 62162005, 61763003 and U21A20474) and the Innovation Project of Guangxi Graduate Education (XYCSZ2022020), Research Fund of Guangxi Key Lab of Multi-source Information Mining \& Security (No. 19-A-02-01), Guangxi 1000-Plan of Training Middle-aged/Young Teachers in Higher Education Institutions, Guangxi ``Bagui Scholar'' Teams for Innovation and Research Project, Guangxi Talent Highland Project of Big Data Intelligence and Application, Guangxi Collaborative Innovation Center of Multisource Information Integration and Intelligent Processing.

\bibliographystyle{IEEEtran}
\bibliography{reference}

\begin{thebibliography}{10}
\providecommand{\url}[1]{#1}
\csname url@samestyle\endcsname
\providecommand{\newblock}{\relax}
\providecommand{\bibinfo}[2]{#2}
\providecommand{\BIBentrySTDinterwordspacing}{\spaceskip=0pt\relax}
\providecommand{\BIBentryALTinterwordstretchfactor}{4}
\providecommand{\BIBentryALTinterwordspacing}{\spaceskip=\fontdimen2\font plus
\BIBentryALTinterwordstretchfactor\fontdimen3\font minus
  \fontdimen4\font\relax}
\providecommand{\BIBforeignlanguage}[2]{{%
\expandafter\ifx\csname l@#1\endcsname\relax
\typeout{** WARNING: IEEEtran.bst: No hyphenation pattern has been}%
\typeout{** loaded for the language `#1'. Using the pattern for}%
\typeout{** the default language instead.}%
\else
\language=\csname l@#1\endcsname
\fi
#2}}
\providecommand{\BIBdecl}{\relax}
\BIBdecl

\bibitem{pengCIKM2019}
Y.~He, Y.~Song, J.~Li, C.~Ji, J.~Peng, and H.~Peng, ``Hetespaceywalk: {A}
  heterogeneous spacey random walk for heterogeneous information network
  embedding,'' in \emph{Proceedings of the 28th ACM International Conference on
  Information and Knowledge Management}.\hskip 1em plus 0.5em minus 0.4em\relax
  {ACM}, 2019.

\bibitem{Diifnet++}
L.~Wu, J.~Li, P.~Sun, R.~Hong, Y.~Ge, and M.~Wang, ``Diffnet++: A neural
  influence and interest diffusion network for social recommendation,''
  \emph{IEEE Transactions on Knowledge and Data Engineering}, 2020.

\bibitem{GNNRes2019}
W.~Fan, Y.~Ma, Q.~Li, Y.~He, Y.~E. Zhao, J.~Tang, and D.~Yin, ``Graph neural
  networks for social recommendation,'' in \emph{The world wide web
  conference}.\hskip 1em plus 0.5em minus 0.4em\relax ACM, 2019.

\bibitem{hg1-2020}
Y.~Dong, Z.~Hu, K.~Wang, Y.~Sun, and J.~Tang, ``Heterogeneous network
  representation learning,'' in \emph{Proceedings of the Twenty-Ninth
  International Joint Conference on Artificial Intelligence}.\hskip 1em plus
  0.5em minus 0.4em\relax ijcai.org, 2020.

\bibitem{SurveyGNNRecSys2021}
C.~Gao, X.~Wang, X.~He, and Y.~Li, ``Graph neural networks for recommender
  system,'' in \emph{Proceedings of the Fifteenth ACM International Conference
  on Web Search and Data Mining}, 2021.

\bibitem{hgmagnn-2020}
X.~Fu, J.~Zhang, Z.~Meng, and I.~King, ``{MAGNN:} metapath aggregated graph
  neural network for heterogeneous graph embedding,'' in \emph{Proceedings of
  The Web Conference 2020}.\hskip 1em plus 0.5em minus 0.4em\relax {ACM} /
  {IW3C2}, 2020.

\bibitem{HGT2020}
Z.~Hu, Y.~Dong, K.~Wang, and Y.~Sun, ``Heterogeneous graph transformer,'' in
  \emph{Proceedings of The Web Conference 2020}.\hskip 1em plus 0.5em minus
  0.4em\relax {ACM} / {IW3C2}, 2020.

\bibitem{pengTKDE2021}
J.~Li, H.~Peng, Y.~Cao, Y.~Dou, H.~Zhang, P.~S. Yu, and L.~He, ``Higher-order
  attribute-enhancing heterogeneous graph neural networks,'' \emph{IEEE
  Transactions on Knowledge and Data Engineering}, 2021.

\bibitem{HINCIKM2021}
L.~Luo, Y.~Fang, X.~Cao, X.~Zhang, and W.~Zhang, ``Detecting communities from
  heterogeneous graphs: {A} context path-based graph neural network model,'' in
  \emph{Proceedings of the 30th ACM International Conference on Information \&
  Knowledge Management}.\hskip 1em plus 0.5em minus 0.4em\relax {ACM}, 2021.

\bibitem{HINAAAI2021}
J.~Zhao, X.~Wang, C.~Shi, B.~Hu, G.~Song, and Y.~Ye, ``Heterogeneous graph
  structure learning for graph neural networks,'' in \emph{Proceedings of the
  AAAI Conference on Artificial Intelligence}.\hskip 1em plus 0.5em minus
  0.4em\relax {AAAI} Press, 2021.

\bibitem{HighOrderRec2020}
Y.~Liu, C.~Liang, X.~He, J.~Peng, Z.~Zheng, and J.~Tang, ``Modelling high-order
  social relations for item recommendation,'' \emph{IEEE Transactions on
  Knowledge and Data Engineering}, 2020.

\bibitem{RecAgainstAI2021}
S.~Zhang, H.~Yin, T.~Chen, Z.~Huang, L.~Cui, and X.~Zhang, ``Graph embedding
  for recommendation against attribute inference attacks,'' in
  \emph{Proceedings of the Web Conference 2021}.\hskip 1em plus 0.5em minus
  0.4em\relax {ACM} / {IW3C2}, 2021.

\bibitem{SocialNetworkPrivacy2020}
H.~Li, Q.~Chen, H.~Zhu, D.~Ma, H.~Wen, and X.~S. Shen, ``Privacy leakage via
  de-anonymization and aggregation in heterogeneous social networks,''
  \emph{{IEEE} Trans. Dependable Secur. Comput.}, vol.~17, no.~2, 2020.

\bibitem{HypergraphCNRec2021}
J.~Yu, H.~Yin, J.~Li, Q.~Wang, N.~Q.~V. Hung, and X.~Zhang, ``Self-supervised
  multi-channel hypergraph convolutional network for social recommendation,''
  in \emph{Proceedings of the Web Conference 2021}.\hskip 1em plus 0.5em minus
  0.4em\relax {ACM} / {IW3C2}, 2021.

\bibitem{HIN2Vec2017}
T.~Fu, W.~Lee, and Z.~Lei, ``Hin2vec: Explore meta-paths in heterogeneous
  information networks for representation learning,'' in \emph{Proceedings of
  the 2017 ACM on Conference on Information and Knowledge Management}.\hskip
  1em plus 0.5em minus 0.4em\relax {ACM}, 2017.

\bibitem{metapath2vec2017}
Y.~Dong, N.~V. Chawla, and A.~Swami, ``metapath2vec: Scalable representation
  learning for heterogeneous networks,'' in \emph{Proceedings of the 23rd ACM
  SIGKDD international conference on knowledge discovery and data
  mining}.\hskip 1em plus 0.5em minus 0.4em\relax {ACM}, 2017.

\bibitem{sun1www21}
Q.~Sun, J.~Li, H.~Peng, J.~Wu, Y.~Ning, P.~S. Yu, and L.~He, ``{SUGAR:}
  subgraph neural network with reinforcement pooling and self-supervised mutual
  information mechanism,'' in \emph{Proceedings of the Web Conference
  2021}.\hskip 1em plus 0.5em minus 0.4em\relax {ACM} / {IW3C2}, 2021.

\bibitem{AGE2021arxiv}
J.~Li, X.~Fu, H.~Peng, S.~Wang, S.~Zhu, Q.~Sun, P.~S. Yu, and L.~He, ``A robust
  and generalized framework for adversarial graph embedding,'' \emph{CoRR},
  vol. abs/2105.10651, 2021.

\bibitem{sun2aaai2022}
Q.~Sun, J.~Li, H.~Peng, J.~Wu, X.~Fu, C.~Ji, and P.~S. Yu, ``Graph structure
  learning with variational information bottleneck,'' in \emph{Proceedings of
  the AAAI Conference on Artificial Intelligence}.\hskip 1em plus 0.5em minus
  0.4em\relax {AAAI} Press, 2022.

\bibitem{pengTNNLS2022}
C.~Li, H.~Peng, J.~Li, L.~Sun, L.~Lyu, L.~Wang, P.~S. Yu, and L.~He, ``Joint
  stance and rumor detection in hierarchical heterogeneous graph,''
  \emph{{IEEE} Trans. Neural Networks Learn. Syst.}, vol.~33, no.~6, 2022.

\bibitem{sun3CIKM2022}
Q.~Sun, J.~Li, H.~Yuan, X.~Fu, H.~Peng, C.~Ji, Q.~Li, and P.~S. Yu,
  ``Position-aware structure learning for graph topology-imbalance by relieving
  under-reaching and over-squashing,'' in \emph{Proceedings of the 31th ACM
  International Conference on Information and Knowledge Management}.\hskip 1em
  plus 0.5em minus 0.4em\relax {ACM}, 2022.

\bibitem{res2-2019}
Y.~Xu, Y.~Zhu, Y.~Shen, and J.~Yu, ``Learning shared vertex representation in
  heterogeneous graphs with convolutional networks for recommendation,'' in
  \emph{Proceedings of the Twenty-Eighth International Joint Conference on
  Artificial Intelligence}.\hskip 1em plus 0.5em minus 0.4em\relax ijcai.org,
  2019.

\bibitem{res3-2019}
Z.~Wang, H.~Liu, Y.~Du, Z.~Wu, and X.~Zhang, ``Unified embedding model over
  heterogeneous information network for personalized recommendation,'' in
  \emph{Proceedings of the Twenty-Eighth International Joint Conference on
  Artificial Intelligence}.\hskip 1em plus 0.5em minus 0.4em\relax ijcai.org,
  2019.

\bibitem{pengSIGIR2020}
J.~Gong, S.~Wang, J.~Wang, W.~Feng, H.~Peng, J.~Tang, and P.~S. Yu,
  ``Attentional graph convolutional networks for knowledge concept
  recommendation in moocs in a heterogeneous view,'' in \emph{Proceedings of
  the 43rd International ACM SIGIR Conference on Research and Development in
  Information Retrieval}.\hskip 1em plus 0.5em minus 0.4em\relax {ACM}, 2020.

\bibitem{hg2-2020}
C.~Yang, Y.~Xiao, Y.~Zhang, Y.~Sun, and J.~Han, ``Heterogeneous network
  representation learning: Survey, benchmark, evaluation, and beyond,''
  \emph{CoRR}, vol. abs/2004.00216, 2020.

\bibitem{pengCIKM2022}
J.~Ren, L.~Jiang, H.~Peng, L.~Lyu, Z.~Liu, C.~Chen, J.~Wu, X.~Bai, and P.~S.
  Yu, ``Cross-network social user embedding with hybrid differential privacy
  guarantees,'' in \emph{Proceedings of the 31th ACM International Conference
  on Information and Knowledge Management}.\hskip 1em plus 0.5em minus
  0.4em\relax {ACM}, 2022.

\bibitem{trajectoriesProtect2021}
J.~Li and G.~Chen, ``A personalized trajectory privacy protection method,''
  \emph{Comput. Secur.}, vol. 108, 2021.

\bibitem{locationAndSemanticPrivacy2021}
B.~Bostanipour and G.~Theodorakopoulos, ``Joint obfuscation of location and its
  semantic information for privacy protection,'' \emph{Comput. Secur.}, vol.
  107, 2021.

\bibitem{locationAndSemanticDP2019}
Y.~Li, X.~Cao, Y.~Yuan, and G.~Wang, ``Privsem: Protecting location privacy
  using semantic and differential privacy,'' \emph{World Wide Web}, vol.~22,
  no.~6, 2019.

\bibitem{obfuscationSemantic2021}
B.~Bostanipour and G.~Theodorakopoulos, ``Joint obfuscation of location and its
  semantic information for privacy protection,'' \emph{Comput. Secur.}, vol.
  107, 2021.

\bibitem{SurveyHereData2021}
M.~Cunha, R.~Mendes, and J.~P. Vilela, ``A survey of privacy-preserving
  mechanisms for heterogeneous data types,'' \emph{Comput. Sci. Rev.}, vol.~41,
  2021.

\bibitem{DPGGAN2021}
C.~Yang, H.~Wang, K.~Zhang, L.~Chen, and L.~Sun, ``Secure deep graph generation
  with link differential privacy,'' in \emph{Proceedings of the Thirtieth
  International Joint Conference on Artificial Intelligence}.\hskip 1em plus
  0.5em minus 0.4em\relax ijcai.org, 2021.

\bibitem{VGAE2016}
T.~N. Kipf and M.~Welling, ``Variational graph auto-encoders,'' in \emph{Neural
  Information Processing Systems Workshop on Bayesian Deep Learning}, 2016.

\bibitem{HetGNN2019}
C.~Zhang, D.~Song, C.~Huang, A.~Swami, and N.~V. Chawla, ``Heterogeneous graph
  neural network,'' in \emph{Proceedings of the 25th ACM SIGKDD international
  conference on knowledge discovery \& data mining}.\hskip 1em plus 0.5em minus
  0.4em\relax {ACM}, 2019.

\bibitem{RGCN2018}
M.~S. Schlichtkrull, T.~N. Kipf, P.~Bloem, R.~van~den Berg, I.~Titov, and
  M.~Welling, ``Modeling relational data with graph convolutional networks,''
  in \emph{European semantic web conference}, ser. Lecture Notes in Computer
  Science, vol. 10843.\hskip 1em plus 0.5em minus 0.4em\relax Springer, 2018.

\bibitem{GCMC2017}
R.~van~den Berg, T.~N. Kipf, and M.~Welling, ``Graph convolutional matrix
  completion,'' \emph{CoRR}, vol. abs/1706.02263, 2017.

\bibitem{RecoGCN2019}
F.~Xu, J.~Lian, Z.~Han, Y.~Li, Y.~Xu, and X.~Xie, ``Relation-aware graph
  convolutional networks for agent-initiated social e-commerce
  recommendation,'' in \emph{Proceedings of the 28th ACM international
  conference on information and knowledge management}.\hskip 1em plus 0.5em
  minus 0.4em\relax {ACM}, 2019.

\bibitem{PersonalPrivacy2010}
M.~Yuan, L.~Chen, and P.~S. Yu, ``Personalized privacy protection in social
  networks,'' \emph{Proc. {VLDB} Endow.}, vol.~4, no.~2, 2010.

\bibitem{DPSGD2016}
M.~Abadi, A.~Chu, I.~J. Goodfellow, H.~B. McMahan, I.~Mironov, K.~Talwar, and
  L.~Zhang, ``Deep learning with differential privacy,'' in \emph{Proceedings
  of the 2016 ACM SIGSAC conference on computer and communications
  security}.\hskip 1em plus 0.5em minus 0.4em\relax {ACM}, 2016.

\bibitem{HDPviaGraph2022}
S.~Torkamani, J.~B. Ebrahimi, P.~Sadeghi, R.~G.~L. D'Oliveira, and
  M.~M{\'{e}}dard, ``Heterogeneous differential privacy via graphs,'' in
  \emph{2022 IEEE International Symposium on Information Theory (ISIT)}.\hskip
  1em plus 0.5em minus 0.4em\relax {IEEE}, 2022.

\bibitem{DP2006}
C.~Dwork, ``Differential privacy,'' in \emph{Automata, Languages and
  Programming, 33rd International Colloquium}, ser. Lecture Notes in Computer
  Science, vol. 4052.\hskip 1em plus 0.5em minus 0.4em\relax Springer, 2006.

\bibitem{DP20062}
C.~Dwork, F.~McSherry, K.~Nissim, and A.~D. Smith, ``Calibrating noise to
  sensitivity in private data analysis,'' in \emph{Theory of cryptography
  conference}, ser. Lecture Notes in Computer Science, vol. 3876.\hskip 1em
  plus 0.5em minus 0.4em\relax Springer, 2006.

\bibitem{GaussianDP2014}
C.~Dwork and A.~Roth, ``The algorithmic foundations of differential privacy,''
  \emph{Found. Trends Theor. Comput. Sci.}, vol.~9, no. 3-4, 2014.

\bibitem{Attention2017}
A.~Vaswani, N.~Shazeer, N.~Parmar, J.~Uszkoreit, L.~Jones, A.~N. Gomez,
  L.~Kaiser, and I.~Polosukhin, ``Attention is all you need,'' in
  \emph{Advances in Neural Information Processing Systems}, 2017.

\bibitem{HGCN2020}
L.~Yu, L.~Sun, B.~Du, C.~Liu, W.~Lv, and H.~Xiong, ``Hybrid micro/macro level
  convolution for heterogeneous graph learning,'' \emph{CoRR}, vol.
  abs/2012.14722, 2020.

\bibitem{graphsage2017}
W.~L. Hamilton, Z.~Ying, and J.~Leskovec, ``Inductive representation learning
  on large graphs,'' in \emph{Advances in Neural Information Processing
  Systems}, 2017.

\bibitem{van2008visualizing}
L.~Van~der Maaten and G.~Hinton, ``Visualizing data using t-sne.''
  \emph{Journal of machine learning research}, vol.~9, no.~11, 2008.

\end{thebibliography}

\end{document}